\documentclass{article}
\usepackage{ijcai26}

\usepackage{xspace}
\usepackage{amsmath}
\allowdisplaybreaks
\usepackage{amssymb}
\usepackage{amsthm}
\usepackage{xcolor}

\usepackage[bb=dsserif]{mathalpha}

\newcommand{\buchi}{B\"uchi\xspace}

\newcommand{\mdp}{\mathcal{M}}
\newcommand{\state}{\mathcal{S}}
\newcommand{\action}{\mathcal{A}}
\newcommand{\kernel}{\mathcal{P}}
\newcommand{\initialstate}{\ensuremath{\mu}}
\newcommand{\reward}{r}
\newcommand{\discount}{\ensuremath{\gamma}}
\newcommand{\E}{\mathbb{E}}
\newcommand{\given}{\,|\,}

\newcommand{\trajectory}{\tau}              
\newcommand{\Vpi}{V^{\pi}}                  
  
\newcommand{\mdplabel}{\mathcal{L}}

\newcommand{\ltl}{\text{LTL}\xspace}
\newcommand{\ltlU}{\textsf{U}}
\newcommand{\ltlX}{\textsf{X}}

\newcommand{\ltlG}{\textsf{G}}

\newcommand{\ltlF}{\textsf{F}}

\newcommand{\ap}{\textsf{AP}}
\newcommand{\alphabet}{\Sigma}
\newcommand{\ttrue}{{\mathit{tt}}}
\newcommand{\ffalse}{\mathit{ff}}
\newcommand{\flang}[1]{[#1]}

\renewcommand{\ltlX}{\operatorname{\mathbf{X}}}
\renewcommand{\ltlF}{\operatorname{\mathbf{F}}}
\renewcommand{\ltlG}{\operatorname{\mathbf{G}}}
\renewcommand{\ltlU}{\mathbin{\mathbf{U}}}

\newcommand{\ltlGF}{\operatorname{\mathbf{GF}}}
\newcommand{\ltlFG}{\operatorname{\mathbf{FG}}}

\newcommand{\ldba}{\mathcal{B}}

\newcommand{\autAcc}{F}
\newcommand{\autTrans}{\delta}
\newcommand{\autState}{Q}

\newcommand{\trueness}{\textit{tr}}
\newcommand{\obset}{\textit{ob}}
\newcommand{\prog}{\textit{prog}}

\usepackage{tikz}
\usetikzlibrary{automata,shapes,calc,fit,arrows,arrows.meta,backgrounds,colorbrewer}
\usepackage{todonotes}

\usepackage{multirow}
\newcommand{\timeout}{\textbf{\textcolor{gray}{TO}}}

\usepackage{times}
\usepackage{soul}
\usepackage{url}
\usepackage[hidelinks]{hyperref}
\usepackage[utf8]{inputenc}
\usepackage[small]{caption}
\usepackage{graphicx}
\usepackage{amsmath}
\usepackage{amsthm}
\usepackage{booktabs}
\usepackage{algorithm}
\usepackage{algorithmic}
\usepackage{comment}
\usepackage[switch]{lineno}

\usepackage[misc]{ifsym}

\newcommand{\daggerfootnote}[1]{%
  \let\oldthefootnote=\thefootnote
  \renewcommand{\thefootnote}{\dagger}
  \footnotetext{#1}
  \let\thefootnote=\oldthefootnote
}

\usepackage{subcaption}

\newtheorem{theorem}{Theorem}
\theoremstyle{remark}

\newtoggle{cameraready}
\toggletrue{cameraready}
\togglefalse{cameraready}

\title{Semantically Labelled Automata for Multi-Task Reinforcement Learning with LTL Instructions\iftoggle{cameraready}{\footnote{For full version and appendicies see \protect\cite{arxivVersion} }}{}}

\author{
Alessandro Abate$^1$ \and
Giuseppe De Giacomo$^1$\and
Mathias Jackermeier$^{1,\text{*}}$\and \\
Jan Kret{\'{\i}}nsk{\'{y}}$^{2,3}$\and
Maximilian Prokop$^{2,3,\text{*}}$\And
Christoph Weinhuber$^{1,\text{*}}$\\
\affiliations
$^1$ University of Oxford\\
$^2$ Technical University of Munich\\
$^3$ Masaryk University Brno\\
}

\begin{document}

\maketitle 

\begin{abstract}
    We study multi-task reinforcement learning (RL), a setting in which an agent learns a single, universal policy capable of generalising to arbitrary, possibly unseen tasks.
    We consider tasks specified as linear temporal logic (LTL) formulae, which are commonly used in formal methods to specify properties of systems, and have recently been successfully adopted in RL\@.
    In this setting, we present a novel task embedding technique leveraging a new generation of \textit{semantic} LTL-to-automata translations, originally developed for temporal synthesis.
    The resulting \textit{semantically labelled} automata contain rich, structured information in each state that allow us to (i) compute the automaton efficiently on-the-fly, (ii) extract expressive task embeddings used to condition the policy, and (iii) naturally support full LTL\@.
    Experimental results in a variety of domains demonstrate that our approach achieves state-of-the-art performance and is able to scale to complex specifications where existing methods fail.
\end{abstract}

\section{Introduction}
Developing reinforcement learning (RL) agents capable of executing complex, temporally extended tasks, such as those expressed in Linear Temporal Logic (LTL)~\cite{DBLP:conf/focs/Pnueli77}, is an active area of research in artificial intelligence. 
A long line of existing research investigates training RL agents to execute a single LTL instruction~\cite{DBLP:conf/cdc/SadighKCSS14,DBLP:conf/atva/BrazdilCCFKKPU14,DBLP:conf/tacas/HahnPSSTW19,DBLP:conf/qest/HasanbeigKA22,DBLP:conf/icml/Voloshin0Y23}.
Traditionally, the inherent memory requirement of LTL specifications is addressed by constructing an automaton (typically a variant of a \buchi automaton) corresponding to the LTL specification.
However, this couples the resulting policy to that specific LTL formula.
Consequently, even minor changes in the objective, e.g.\ altering the items that a warehouse robot needs to collect, requires retraining the agent.
Thus, single-task approaches are ineffective at agent deployment time, where retraining is infeasible, or when task specifications change frequently.

As a result, recent years have seen increased interest in the \textit{multi-task} RL setting \cite{DBLP:journals/corr/OhSLK17}.
Generally speaking, multi-task RL focuses on training a single, generalist policy capable of executing arbitrary tasks that may not be known apriori.
Intuitively, such a policy learns a set of reusable behaviours and how to compose them in order to achieve new tasks.
This is particularly challenging for LTL tasks, as the agent must additionally manage the required memory.

Recent approaches to tackle multi-task RL with LTL specifications can be broadly categorised into \emph{decomposition}-based methods~\cite{DBLP:conf/iclr/LeonSB22,DBLP:conf/nips/0002M023,DBLP:journals/corr/abs-2508-01561}, which break down the objective into several subtasks, and \emph{holistic} approaches that consider the full specification~\cite{DBLP:conf/icml/VaezipoorLIM21,DBLP:journals/ml/XuF24,DBLP:conf/iclr/JackermeierA25}.
While decomposition into subtasks facilitates generalisation by allowing the agent to recombine learned subtask policies, it may lead to suboptimal solutions due to missing context.
For instance, given the task ``approach shelf, then pick up hammer", a warehouse robot focusing only on the first objective is unaware of which shelf to approach.
To avoid these kinds of problems, holistic approaches seek policies that consider the full context of the task.
This, however, is significantly more challenging, as the generalisation must be achieved solely through learning, which necessitates an amenable task embedding. 
In particular, arbitrary tasks must be embedded in a finite-dimensional feature space that preserves the semantic similarity of tasks, while also effectively encoding the necessary memory.

\textbf{Our contribution}.
In this paper, we propose novel task embeddings based on a new generation of LTL-to-automata translations~\cite{DBLP:conf/cav/EsparzaK14,DBLP:journals/jacm/EsparzaKS20}.
Unlike traditional translations, these approaches are \textit{semantic}:
each automaton state is annotated with additional information (so-called \textit{semantic labelling}) that intuitively describes the semantic meaning of that state.
Recent work already exploited semantic labelling successfully in challenging problems from formal methods such as parity game solving \cite{DBLP:conf/atva/KretinskyMM19,DBLP:conf/cav/KretinskyMPR23} and reactive synthesis \cite{DBLP:conf/tacas/KretinskyMPZ25}.
In this work, we demonstrate that semantic labelling enables an effective embedding of arbitrary LTL tasks, while naturally addressing their memory requirements.
In contrast to previous methods, our approach is both (i) applicable to full LTL 
and (ii) lightweight, as it can be computed efficiently on the fly.
Finally, we design a policy around these embeddings and demonstrate that this enables successful generalisation to specifications of previously infeasible complexity.

\section{Preliminaries}
\paragraph{Linear temporal logic.} 
Let $\ap$ be a finite set of atomic propositions.
LTL formulas are generated by $\varphi ::= \ttrue\mid a  \mid \neg \varphi \mid \varphi \land \varphi \mid \ltlX\varphi \mid \varphi \ltlU \varphi$.
where $a\in \ap$, and $\ltlX$ (``Next"), $\ltlU$ (``Until") are temporal operators.
Further, we define the common abbreviations , $\ltlF\varphi\equiv\ttrue\ltlU \varphi$ (``Finally/Eventually") and $\ltlG \varphi \equiv \neg\ltlF\neg\varphi$ (``Globally/Always") \cite{DBLP:conf/focs/Pnueli77}.

An $\omega$-\emph{word} $w$ is an infinite sequence of letters $w[0]\,w[1]\,w[2]\dots$ with $w[i] \in \alphabet = 2^{\ap}$.
We denote the infinite suffix $w[i]\,w[i+1]\dots$ by $w_i$ and present the full satisfaction relation $w\models\varphi$ in Appendix \ref{appendix:ltl_sat_relation}.
The \emph{language} of $\varphi$, is the set of all satisfying words, denoted by $\flang{\varphi}$.
\ltl has several fragments, most prominently \emph{Safety} and \emph{Guarantee}.
Intuitively, they denote that something bad never happens and that something good happens eventually, respectively.
Further, we consider \emph{recurrence} (a guarantee holds infinitely often) and \emph{persistence} (a safety formula holds eventually).

\paragraph{Limit-deterministic \buchi automata.}
We define a limit-deterministic \buchi automaton (LDBA) as a tuple $\ldba = (\autState,\alphabet, \autTrans, \mathcal{E} , q_0, \autAcc)$, where $\autState$ is a finite set of states, $\alphabet := 2^{\ap}$ is an alphabet, $\autTrans: \autState\times\alphabet \to \autState$ is a transition function, $\mathcal{E}: Q \to 2^{Q}$ are (non-deterministic) epsilon transitions, $q_0 \in Q$ the initial state, and $\autAcc \subseteq Q$ is a set of accepting states, satisfying the following.
There exists a partitioning $Q= Q_I\uplus Q_A$ such that (i) the transition function never changes the component, (ii) epsilon transitions only exist from $Q_I$ to $Q_A$, (iii) every accepting state is in $Q_A$, and (iv) the initial state is in $Q_I$.

A run on an $\omega$-word $\omega = \sigma_0 \sigma_1 \cdots$ is an infinite sequence of states $q_0 q_1 \cdots$ that begins in the initial state and evolves according to the transition function 
except for at most one index where it may take an $\epsilon$-transition to $Q_A$ and remain in $Q_A$ thereafter.
Let $j$ be the index of the epsilon transition.
Then, formally, for all $i\geq 0$ with $i\not= j$, $q_{i+1} = \autTrans(q_i, \sigma_i)$ and $q_j \in Q_I, q_{j+1} \in \mathcal{E}(q_j)$.
A run is accepting if it visits $\autAcc$ infinitely often and a word is accepted by the LDBA if there exists at least one accepting run.
For any LTL formula $\varphi$ there exists an LDBA that accepts $[\varphi]$, e.g.~\cite{DBLP:conf/cav/SickertEJK16}.

\paragraph{Markov decision processes.} \label{para:mdp}
We define a Markov decision process (MDP) as a tuple
$\mdp = (\state, \action, \kernel, \initialstate, \reward, \discount)$,
where $\state$ is the state space, $\action$ the action space,
$\kernel: \state \times \action \to \Delta(\state)$ the (unknown) transition kernel, 
$\initialstate \in \Delta(\state)$ the initial state distribution, $\reward: \state \times \action \times \state \to \mathbb{R}$ the reward function, and $\discount \in [0,1)$ the discount factor. 
Here $\Delta(X)$ denotes the set of probability distributions over set $X$ and we write $x \sim \mu$ to denote that $x$ is sampled from a distribution $\mu$.
A (memoryless) policy is a map
$\pi:\state\to\Delta(\action)$, from state $s$ to a distribution over actions.
Rolling out $\pi$ in $\mdp$ yields a trajectory $\trajectory=(s_0,a_0,r_0,s_1,a_1,r_1,\ldots)$ with $s_0\sim\initialstate$, $a_t\sim\pi(\cdot\given s_t)$, $s_{t+1}\sim\kernel(s_t,a_t)$, and $r_t=\reward(s_t,a_t,s_{t+1})$.
A policy conditioned on task $\varphi$ is denoted as $\pi\given\varphi$ or $\pi(\cdot\given s_t,\varphi)$ when applied to a concrete state $s_t$.
We interpret MDPs via a labelling map $\mdplabel: \state \to 2^\ap$.
Given a trajectory $\tau$, $\mdplabel$ induces an $\omega$-word $w = \mdplabel(s_0)\mdplabel(s_1)\cdots \in (2^{\ap})^\omega$, on which we evaluate \ltl\@.

\paragraph{Reinforcement learning.}
Following \cite{DBLP:books/lib/SuttonB2018}, the objective of reinforcement learning (RL) is to find a policy $\pi^\star$ that maximises the \emph{expected discounted return} $J(\pi)\;=\;\E_{\trajectory\sim\pi}\!\left[\sum_{t=0}^{\infty}\discount^{t}\,r_t\right]$. We write $\tau\sim\pi$ to emphasise that the trajectory distribution depends on $\pi$. The \emph{state-value function} of a policy $\pi$ is $\Vpi(s) \;=\; \mathbb{E}_{\tau\sim\pi}\!\left[ \sum_{t=0}^{\infty} \gamma^{t} r_t \,\middle|\, s_0=s \right]$ and corresponds to the expected discounted return starting from state $s$ and following policy $\pi$.
In addition, we also consider a specification-driven objective, namely maximizing the probability of satisfying an \ltl task $\varphi$.
For a policy $\pi$, we define this satisfaction probability as $\Pr(\pi \models \varphi) = \mathbb{E}_{\tau \sim \pi}\bigl[\mathbb{1}[\tau\models\varphi]\bigr]$,
where $\mathbb{1}[\trajectory\models\varphi]$ is $1$ if the trace induced by $\trajectory$ satisfies $\varphi$, and $0$ otherwise.

\section{Problem Definition}
We study multi-task RL with LTL specifications, following \cite{DBLP:conf/icml/VaezipoorLIM21,DBLP:conf/iclr/JackermeierA25}.
In contrast to single-task RL, we consider \textit{task-conditioned} policies $\pi\given\varphi$ that receive a \textit{current} task $\varphi$ as input.
Such policies can adapt to a new task $\varphi$ without requiring any retraining.
In particular, the MDP is inacessible and solely $\varphi$ itself may be subject to any computation before and during the execution.
Formally, the goal is to learn a single universal policy that maximises the overall satisfaction probability over some distribution $\mathcal D$ over LTL tasks:

\begin{equation} \label{eq:multi-task-max}
    \pi^{\star}
    \in
    \arg\max_{\pi}
    \mathop{\mathbb{E}}\limits_{\substack{\varphi \sim \mathcal D, \\ \tau \sim \pi\mid\varphi}} \left[ \mathbb{1}[\tau \models \varphi] \right]
\end{equation}
The following two key challenges arise in this setting.

\paragraph{Memory requirements of LTL tasks.}
Notice that \ltl tasks are inherently \textit{non-Markovian}: 
the agent needs to be aware of already observed labels to decide its next action, as exemplified by our running example $\ltlF r \land \ltlFG y$.
In the same MDP state, the agent has to perform different actions depending on whether it has already seen $r$; if it has, it can focus on continuously achieving $y$; otherwise, it needs to satisfy the finite goal $\ltlF r$ first.
While it is in principle possible to learn a non-Markovian policy $\pi (\cdot\given s_0, \dots, s_t, \varphi)$ conditioned on the entire trajectory history, this introduces significant complexity in practice.
Hence most RL approaches with LTL tasks aim to recover a Markovian view by augmenting the MDP state with memory (typically the states of an automaton for $\varphi$) that keeps track of the current task's ``progress'' (intuitively, tracking the already achieved subgoals of $\varphi$).

\paragraph{Universal progress interface.}
In single-task RL, tracking progress of subgoals and conditioning a policy on it is straightforward.
The common approach is to construct an automaton for the given \ltl task, as by design, the automaton states form sufficient memory for deciding what to do next in the given task.
We then retrieve a Markovian setting by training a separate sub policy for each state.
At test time, the conditioning on task progress is achieved by tracking the current automaton state and employing the respective sub-policy.

However, in multi-task RL, conditioning via state-dependent sub-policies is no longer feasible, since the automaton is not fixed a priori, but changes with every new task.
Instead, a universal policy capable of being conditioned on \emph{any} possible task needs to be capable of processing \emph{any} possible automaton.
This necessitates a \textit{universal progress interface} (UPI) which is a finite-dimensional embedding of arbitrary tasks, that the policy can be conditioned on.
A UPI should (i) reduce any tasks to a common, learnable representation (ii) update that representation depending on achieved progress, and (iii) be computationally cheap, as at test time, it needs to be computed online during interaction with the environment.
To achieve the task-conditioned behaviour of the policy, the UPI needs to be semantically meaningful: if two different tasks require similar behaviour, they should have similar embeddings.
For instance, $\ltlF r \land \ltlFG y$ and $\ltlF r$ should have similar embeddings as both require an $r$ to proceed.

\section{Existing Approaches}
We proceed to discuss two current approaches to constructing a UPI which we build upon.

\subsection{Formula Progression Approach}
In \textsc{LTL2Action} \cite{DBLP:conf/icml/VaezipoorLIM21}, the \ltl formula itself is embedded as the universal progress interface.
Crucially, task progress can be updated by so-called \textit{formula progression} \cite{DBLP:conf/aaai/BacchusK96}.
Intuitively, for a formula $\varphi_t$ and a letter $\sigma$ the progressed formula $\prog(\sigma,\varphi_t)$ captures what remains to be satisfied after seeing $\sigma$.
For example, if $\varphi_t = \ltlF(r \land \ltlFG y)$ and $\sigma = \{r\}$, then $\prog(\sigma,\varphi_t) = \ltlFG y$ (see Appendix \ref{app:prog} for a full definition).

\textsc{LTL2Action} then embeds the current (progressed) formula via a graph neural network~(GNN) \cite{DBLP:journals/tnn/ScarselliGTHM09} $g$ that operates on LTL syntax trees.
The GNN is trained end-to-end with the task-conditioned policy $\pi(\cdot\given s_t, g(\varphi_t))$ via RL by issuing a reward of $1$ exactly when $\varphi_t$ progressed to $\ttrue$, which denotes satisfaction of the task.

A benefit of this approach is that formula progression can be computed efficiently and on the fly, i.e.\ each step is only computed when it is actually required.
However, its reliance on eventually progressing to $\ttrue$ to give rewards limits this approach to the co-safety (guarantee) fragment of \ltl.
Giving rewards for general \ltl tasks requires tracking additional information which plain formula progression does not provide.

\subsection{Automata-Theoretic Approach}
The common approach to support full LTL is to convert the current task to an automaton and keep track of the current automaton state as a UPI\@.
This is the approach taken by \textsc{DeepLTL}~\cite{DBLP:conf/iclr/JackermeierA25}, which employs LDBAs as suitable automata.
Formally, the task-conditioned policy $\pi\given\varphi$ then operates over the \textit{product MDP} $\mathcal M_\varphi$, whose state space consists of the MDP state space $\mathcal S$ augmented with the LDBA state space $\mathcal Q$.
Non-deterministic $\varepsilon$-transitions are handled by introducing special $\varepsilon$-actions $\varepsilon_{q'}$ that update the current LDBA state $q$ to $q' \in \mathcal E(q)$ without modifying the MDP state.
Assigning positive rewards of $1$ to visiting accepting states $F$ in the LDBA then enables learning a task-conditioned policy.
Formally, the policy is optimised via the following objective, where $\mathbb 1_F$ is the indicator function of $F$:
\begin{equation}
\label{eq:vanilla-automata-theory-single-rlk}
\pi^{\star} (\cdot\given s_t, q_t)
    \in
    \arg\max_{\pi} 
\mathop{\mathbb{E}}\limits_{\substack{\varphi \sim \mathcal D, \\ \tau \sim \pi\mid\varphi}}
\Bigl[
  \sum_{t'=0}^{\infty} \gamma^{t} \mathbb{1}_F(q_{t'})]
\Bigr].
\end{equation}

The key question in the automata-theoretic approach is how to obtain meaningful features of LDBA states from arbitrary formulae that are not known apriori.
\textsc{DeepLTL} proposes so-called \emph{reach-avoid sequences} as a common representation of automaton states.
Intuitively, these are sequences of letters to reach and avoid that would result in visiting accepting states infinitely often. 
They hence form a complete plan on how to produce a satisfying trace from the current state.
To chose between the many sequences that could represent a state, \textsc{DeepLTL} learns a critic network that judges the feasibility of any sequence in the given MDP\@.
Formally, \textsc{DeepLTL} learns a policy $\pi(\cdot\given s_t, e(\varsigma))$ conditioned on a reach-avoid sequence $\varsigma$ embedded via an embedding network $e$.
At test time, the agent tracks the current state of the LDBA and continually searches for the optimal (according to the critic) reach-avoid sequence $\varsigma^\star$ for representing the current state $q_t$, and executes actions given by $\pi(\cdot\given s_t, e(\varsigma^\star))$. 

While this approach works for full \ltl, the search for reach-avoid sequences is prohibitively expensive as (i) the entire LDBA, needs to be constructed upfront and (ii) all of its accepting paths need to be exhaustively enumerated to find the optimal reach-avoid sequence (see Appendix~\ref{app:complexity} for an exact derivation of the complexity).
This limits \textsc{DeepLTL} to relatively simple tasks, where computing and enumerating sequences of the entire automaton is feasible.

\paragraph{Rationale for LDBAs.} 
The choice of LDBA as automata is important as they combine several useful properties.
First, their \buchi acceptance can be converted naturally to a reward signal s.t.\ maximizing rewards coincides with optimising satisfaction probability \cite{DBLP:conf/tacas/HahnPSSTW19}.
Crucially, this is not possible for other acceptance conditions like Rabin, and impractical for parity \cite{impossibility_result}.
Note however, that this requires slightly more sophisticated objectives than the one we presented in Equation~\ref{eq:vanilla-automata-theory-single-rlk}, in order to address the bias introduced by the discount factor $\gamma$~\cite{DBLP:conf/tacas/HahnPSSTW19,DBLP:conf/icra/Bozkurt0ZP20,DBLP:conf/icml/Voloshin0Y23}.
Further, while some non-determinism is unavoidable when aiming for a \buchi automaton that captures full LTL, the non-determinism of LDBAs is very manageable.
Specifically, their non-determinism is \emph{good-for-MDP},
meaning that it can be resolved by the agent (i.e.\ the epsilon transitions turn into new actions) without sacrificing any optimality of the policy.\footnote{Technically, not every LDBA is good-for-MDPs, but the ones produced by 
\cite{DBLP:conf/cav/SickertEJK16,DBLP:journals/jacm/EsparzaKS20} are.}

\tikzstyle{box} = [draw,rectangle,rounded corners=5pt,minimum size=1cm]
\tikzstyle{autstate} = [draw,circle,inner sep=2pt,minimum size=.4cm]
\definecolor{trueness_color}{HTML}{CF3619}
\definecolor{attention_color}{HTML}{0D97A8}
\definecolor{currState_color}{HTML}{f78400}
\definecolor{autState_color}{HTML}{F2F099}
\definecolor{autState_color2}{HTML}{F58C78}
\definecolor{autState_color3}{HTML}{A9FC9D}
\definecolor{autState_color4}{HTML}{1CF4FF}
\definecolor{autState_color5}{HTML}{CE85FF}

\begin{figure*}
    \centering
    \begin{tikzpicture}[
        >={Stealth[round]}, 
        thick,
        font=\small,
        node distance=0.6cm, 
        box/.style={draw, rectangle, rounded corners=5pt, minimum size=1cm, fill=white},
        autstate/.style={draw, circle, inner sep=2pt, minimum size=.4cm, fill=white},
        section label/.style={font=\bfseries\footnotesize, color=gray!80, align=center},
        annotation/.style={font=\scriptsize\itshape, color=gray!90, align=left}
    ]

    \node (policy) at (11,0) {\Large $\pi(a \given s_t, e_\varphi)$};
    \node (boxExtensionNode) at (9.7,0) {};

    \node at (8,0.7) (colorsememb) {$
        \begin{bmatrix}
        \color{trueness_color} 0.25 \\
        \color{trueness_color} 0.0 \\
        \vdots \\
        \color{attention_color} 0.0 \\
        \color{attention_color} 1.0 \\
        \vdots \\
        \end{bmatrix}
    $};
    \node[section label] at (8,2.59) (lbl_emb) {Embedding $e_\varphi$};

    \node[annotation, align=left] at (8,-1.1) (legend) {
        \textcolor{trueness_color}{$\blacksquare$} Trueness \\
        \textcolor{attention_color}{$\blacksquare$} Attention
    };

    \node[box, inner sep=-1pt, fill=autState_color2!30] at (4.7,0.5) (semLab2) {
        \setlength{\tabcolsep}{3pt}
        \scriptsize
        \begin{tabular}{c|cc}
        \multirow{2}{*}{$q_1$} &
            ~M: & $\phantom{\ltlF r \land } ~~\ltlFG y~$ \\
          &  ~B: & $-$
        \end{tabular}
    };
        \node[above=0.1 of semLab2, box, inner sep=-1pt, fill=autState_color!30, draw=currState_color] (semLab1) {

        \setlength{\tabcolsep}{3pt}
        \scriptsize
        \begin{tabular}{c|cc}
        \multirow{2}{*}{$q_0$} &
           ~M: &  $\ltlF r \land \ltlFG y~$ \\
          &   ~B: & $-$
        \end{tabular}
    };    
    \node[below=0.1cm of semLab2, box, inner sep=-1pt, fill=autState_color3!30] (semLab3) {
        \setlength{\tabcolsep}{3pt}
        \scriptsize
        \begin{tabular}{c|cc}
        \multirow{2}{*}{$q_2$} &
            ~M: &  $\phantom{\ltlF r \land \ltlF}\ltlG y~$ \\
           &  ~B: & $-$
        \end{tabular}
    };

    \node[autstate, fill=autState_color2!50] (q1) at ($(semLab2.west) - (1.3, -0.6)$) {$q_1$};
    \node[autstate, initial, initial text=, fill=autState_color!50, draw=currState_color] (q0) at ($(q1) - (1.2, 0)$) {$q_0$};
    \node[autstate, accepting, fill=autState_color3!50] (q2) at ($(q1) + (0, -1.2)$) {$q_2$};
    \node[autstate] (q3) at ($(q2) - (1.2, 0)$) {$\bot$};

    \node[above=1.3cm of $(q0)!0.5!(q1)$, section label,align=center]  (lbl_aut) {LDBA};
    \node[above=0.2cm of semLab1, section label,align=center] (lbl_aut) {Semantic Labelling};

    \path[->] 
        (q0) edge[] node[above] {$r$} (q1)
        (q1) edge[] node[right] {$\varepsilon$} (q2)
        (q2) edge[] node[above] {$\neg y$} (q3)
        (q0) edge[loop, in=70, out=110,  looseness=5] node[left,near start] {$\neg r$} (q0)
        (q1) edge[loop, in=70, out=110,  looseness=5] node[right,near end] {$\ttrue$} (q1)
        (q2) edge[loop, in=250, out=290,  looseness=5] node[right,near start] {$y$} (q2)
        (q3) edge[loop, in=250, out=290,  looseness=5] node[left,near end] {$\ttrue$} (q3);
        
    \draw[->, dashed, gray!80, shorten >=2pt, shorten <=2pt] (q0.north east) to[bend left=30] node[midway, above, font=\scriptsize, color=gray!80] {} (semLab1.west);
    \draw[->, dashed, gray!80, shorten >=2pt, shorten <=2pt] (q1.south east) to[bend right=25] node[midway, above, font=\scriptsize, color=gray!80] {} (semLab2.west);
    \draw[->, dashed, gray!80, shorten >=2pt, shorten <=2pt] (q2.south east) to[bend right=15] node[midway, above, font=\scriptsize, color=gray!80] {} (semLab3.west);

    \node at (15,0.5) (envImg) {
        \includegraphics[width=2.6cm]{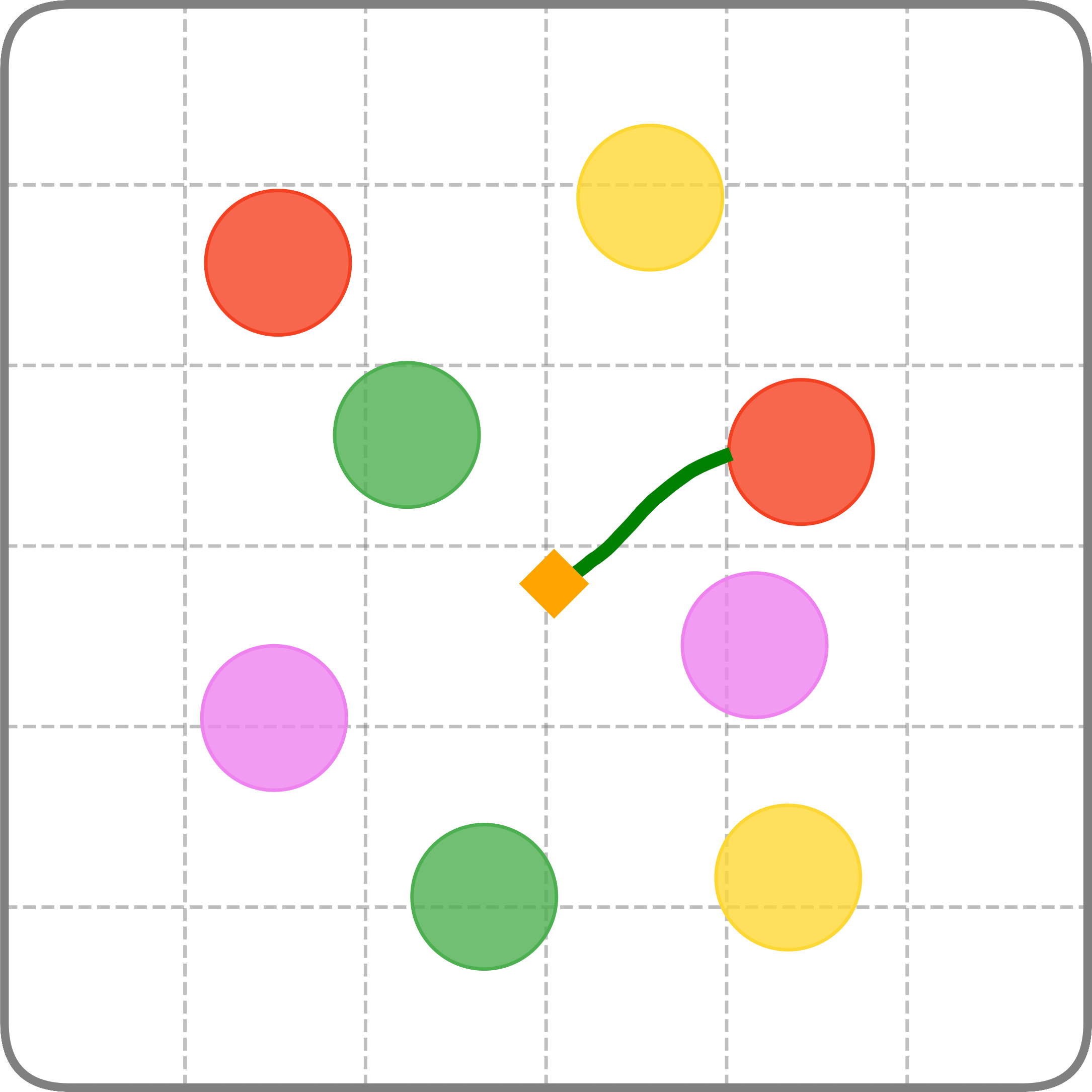}
    };
    \node[above=0.35cm of envImg, section label,align=center] (lbl_env) {Environment};

    \path[->, gray!50, line width=1.5pt] 
    (semLab1) edge [in=west,out=east] node[below=0.5cm,pos=0.7, black, font=\scriptsize] {EMBED} ($(colorsememb.west) + (0,-0.7)$);
    \draw[->, gray!50, line width=1.5pt] ($(colorsememb.east) + (0,-0.7)$) -- (policy.west);

    \draw[->, gray!50, line width=1.5pt] ($(envImg.west) + (0,-0.5)$) -- node[above,align=center, black] {state $s_t$}  (policy.east);

    \draw [->,gray!50, line width=1.5pt] 
    (envImg.south) -- +(0,-.62) -| node[above,pos=0.1,black] {MDP label $\{\textbf{r}\textit{ed}\}$} node[above left,pos=0.6,black, font=\scriptsize] {UPDATE} ($(q0)!0.5!(q1) + (0,-2)$);

    \draw [->,gray!50, line width=1.5pt] ($(policy.north)+(-0.22,0)$) |- node[above,pos=0.7,black] {action $a$}
    ($(envImg.west) + (0,0.5)$);

    \begin{scope}[on background layer]
        \node[fit=(q0) (q3) (lbl_aut) (semLab2) (colorsememb) (lbl_emb) (legend), rounded corners, fill=gray!5, draw=gray!20, dashed, inner sep=8pt, name=grey_box] {};
        
        \node[fit=(boxExtensionNode) (policy) (envImg) (envImg |- grey_box.north) (envImg |- grey_box.south) (policy |- grey_box.east), rounded corners, fill=blue!4, draw=blue!15, dashed, inner sep=0pt] {};
    \end{scope}

    \end{tikzpicture}
    \caption{Schematic overview of our approach for the ZoneEnv environment (a robot navigating a 2D plane with zones of different colours) and the LTL task $\varphi=\ltlF \textbf{r}\textit{ed} \land \ltlFG \textbf{y}\textit{ellow}$.
    We first convert the task to an LDBA with semantic labelling (see Section \ref{sec:semanticLabelling}).
    Note that initially only $q_0$ is constructed.
    Then the semantic labelling of the current state $q_0$ is embedded and passed to the policy (see Section \ref{sec:embedding}).
    Together with the current state of the MDP $s_t$, the policy yields action $a$ in the MDP.
    The robot's trajectory in the MDP (green line) enters a zone with MDP label $\textbf{r}\textit{ed}$ which is used to update the automaton.
    Only then, the corresponding successor $q_1$ is constructed in the automaton.
    }
    \label{fig:overview}

\end{figure*}

\subsection{Other Related Work} 
Several other works recently addressed the problem of multi-task RL with \ltl specifications.
Many of these approaches decompose the \ltl specification into sub-tasks, which are then completed one-by-one~\cite{DBLP:conf/iclr/LeonSB22,DBLP:conf/nips/0002M023,DBLP:conf/icra/LiuSRJ0T24,DBLP:journals/corr/abs-2508-01561}.
However, this generally still requires constructing the entire automaton upfront and furthermore, can lead to \textit{myopic} solutions due to missing context.
In contrast, we specifically design our approach to not suffer from such suboptimality.
\citeauthor{DBLP:journals/ml/XuF24}~[\citeyear{DBLP:journals/ml/XuF24}] propose a non-myopic approach based on the option framework~\cite{DBLP:journals/ai/SuttonPS99}, but this is limited to a small fragment of \ltl.
Recently, \citeauthor{ltlgnn}~[\citeyear{ltlgnn}]
investigated GNNs to improve the representation learning capabilities of \textsc{DeepLTL}, but their approach suffers from the same fundamental drawbacks as \textsc{DeepLTL}\@.

\section{Semantically Labelled Automata Approach}
We start by giving an overview of our approach.
We follow the automata-theoretic approach in the sense that we construct an LDBA for the current task, update its state to track progress and condition our policy on the current state.
The crux is to compute a meaningful embedding of any possible automaton state for any possible LTL formula (and ideally just by looking at that state alone, without considering the entire automaton).
Since there are infinitely many automaton states and we can only observe finitely many in training, the key question is how to embed the states into a structured feature space amenable to learning, i.e.\ where semantically similar states are close to each other.
Traditionally, states of (limit-)deterministic automata for LTL formulae bear no logical structure and are simply identified with a numeric index produced by the translation algorithm, so this may seem ambitious.
However, we employ more modern, \emph{semantic translations} to construct the LDBA, which are the product of a long line of research on constructing small automata \cite{DBLP:conf/cav/KretinskyE12,DBLP:conf/cav/EsparzaK14,DBLP:journals/jacm/EsparzaKS20}.
The key property of these translations is that each state is identified with a simple structure (e.g.\ a pair or a vector) of LTL formulae.
This so-called \emph{semantic labelling} fully describes the semantics of a state (that is, which subgoals remain to be achieved from here on) and thus, embedding this information can inform the policy about progress of the task and remaining obligations.
Crucially, the semantic labelling of a state is available without having to compute the entire automaton.
Instead, required parts of the automaton can be computed on-demand, depending on the observed MDP labels.
See Figure \ref{fig:overview} for an illustration of this approach.

\subsection{Semantic Labelling}\label{sec:semanticLabelling}
The semantic labelling of our automaton states originates from the translation presented in \cite{DBLP:journals/jacm/EsparzaKS20}, which is implemented in \textsc{Owl} \cite{Owl}.
Their LDBAs require keeping track of only two LTL formulae, which is achieved through upfront normalisation \cite{LTL_Normalization}.
First, we have the \textbf{M}\textit{ain formula}, which is the entire formula in the initial state and for every subsequent state is derived using formula progression.
Intuitively, the main formula can be thought of as the ``remaining language" from the current state onwards.
However, for a recurrence formula like $\ltlGF r$ (i.e.\ reach $r$ infinitely often), the remaining obligation will never change and thus, just tracking the main formula cannot distinguish between words that satisfy this property and ones that do not.
Thus, a second formula, the \textbf{B}\textit{reakpoint formula} is required to track the inner formula of recurring obligations ($\ltlF r$ in this case).
Whenever this formula progresses to $\ttrue$, it resets and emits a \buchi acceptance signal.
That way, seeing a \buchi signal infinitely often corresponds directly to satisfying the inner formula infinitely often.

As an example, consider the automaton of Figure \ref{fig:overview}.
The main formula of $q_0$ is the entire task and once we observe $r$, the automaton updates to state $q_1$ in which the main formula has progressed to $\ltlFG y$.
The $\varepsilon$-transition guesses when the safety formula $\ltlG y$ starts to hold and $q_2$ verifies that guess through formula progression.
This example has no breakpoint formulae due to the lack of recurrence formulae; see Appendix \ref{app:breakpoint} for an example with a breakpoint formula.

\subsection{Embedding the Semantic Labelling}\label{sec:embedding}
Embedding the semantic labelling requires capturing the relevant information of two LTL formulae, which we tackle individually and concatenate the result.
Recall that the embedding space needs to create numerical proximity of semantically similar states (or rather similar formulae) to enable generalisation.
Intuitively, the similarity of two formulae is proportional to how much their languages overlap.
Hence, our features intuitively approximate a formula's language by probing the effects of letters when seen at the current- or subsequent steps.
Additionally, we employ simple complexity measures of formulae (e.g.\ the height of the syntax tree as suggested in \cite{DBLP:conf/cav/KretinskyMPR23}) to allow the agent to choose less complex states when resolving $\varepsilon$-transitions.

\paragraph{Exploiting trueness.}
This feature approximates the immediate effect of observing a letter by measuring how much each letter progresses the formula towards or away from satisfaction.
We thus need a measure of proximity to satisfaction.
Intuitively, it is clear that $\ltlFG y$ is closer to satisfaction than $\ltlF r \land\ltlFG y$, since the former already achieved a subgoal.
To capture this, we resort to \emph{trueness} which was first introduced in \cite{DBLP:conf/atva/KretinskyMM19}.
Intuitively, trueness measures how easy it is to satisfy an \ltl formula, by propositional approximation of the temporal behaviour.
Formally, to compute the trueness $\trueness(\varphi)$ of a formula $\varphi$ we replace all temporal subformulae $\psi$ by new propositional variables $x_\psi$ and compute the ratio of satisfying assignments to total assignments in the resulting propositional formula.
For example for $\ltlF r \land \ltlFG y$, we replace all temporal subformulae, and get $x_{\ltlF r} \land x_{\ltlFG y}$.
In the resulting conjunction, 1 out of 4 assignments satisfy the formula and thus $\trueness(\ltlF r \land \ltlFG y) = 0.25$.
Further, we have $\trueness(\ltlFG y) = 0.5$, which indicates that the latter formula is easier to satisfy and progressing to the latter is desirable.

We use trueness to compute a feature value for every MDP label of $\mdplabel$.
As we are interested in the impact of a letter, we compare the change in trueness before and after seeing that letter, i.e.\ after progressing the formula using that letter.
Formally, for a formula $\varphi$ and a letter $\sigma$ the trueness feature $f_{tr}$ is defined as: 
\[
f_{tr}(\varphi,\sigma)= \trueness(\prog(\varphi,\sigma))-\trueness(\varphi).
\]

In the embedding vector of Figure \ref{fig:overview}, the first two red numbers correspond to the feature values $f_{tr}(\varphi,\{r\})$ and $f_{tr}(\varphi,\{y\})$, where $\varphi =\ltlF r \land \ltlFG y$.
By looking at these two numbers, the agent can infer that an $\{r\}$ would make progress while seeing $\{y\}$ currently has no effect.
Symmetrically to progress, this feature also captures violations.
E.g.\ in $q_2$ where the main formula is $\ltlG y$ we have the feature value $f_{tr}(\ltlG y,\{\})=-0.5$, indicating that not seeing a $y$ (terminally) hurts the progress towards satisfying $\ltlG y$.
We further refine this feature be employing different normalisations (e.g.\ among all letters, or among all positive/negative values) that result in slightly different semantic meanings.
For a full list and explanation of these, we refer to Appendix \ref{App:features}.
Ultimately, despite being unable to capture deep temporal relations (e.g.\ $\ltlF r \land \ltlG \neg r$ having trueness $0.25$ despite being unsatisfiable), trueness yields a great propositional approximation of task-progress, especially for formulae appearing in practice.

\paragraph{Propositional attention.}
Trueness alone yields a myopic embedding, meaning that it only considers the immediate effect of seeing a letter while disregarding what happens afterwards.
For example, the trueness feature values of $\ltlF r \land \ltlFG y$ alone, do not yield any indication that after seeing $\{r\}$, $\{y\}$ will be important, which can result in suboptimal behaviour in certain environments.
To combat this, we propose an attention-like feature, where we relate pairs of propositions by how reliant one is on the other for satisfying the formula.

Formally, for a formula $\varphi$ and a pair of propositions $(p, q)$,
we first compute the formula after seeing $\{p\}$ denoted as $\varphi'$.
Then, we compute the \emph{obligation sets} \cite{10.1109/TIME.2013.19} of $\varphi'$, denoted by $\obset(\varphi')$.
A single obligation set is a letter $\sigma\in2^{AP}$ that if seen infinitely often, satisfies the formula, i.e.\ $\sigma^\omega \models \varphi'$.
We provide the full definition of obligation sets in Appendix \ref{app:obset}, but intuitively, the obligation sets approximate the infinite behaviours of $\varphi'$ on a propositional level.
To compute the feature value, we then count how often $q$ appears in the obligation set (treating $q$ and $\neg q$ separately) and divide by the total number of obligation sets.
Putting it all together, the positive propositional attention feature $f_{att}^+$ for a formula $\varphi$ and a pair of propositions $(p, q)$ is given by
\[
f_{att}^+(\varphi,p,q) = \frac{|\{o\in \obset(\varphi')\mid q \in o \}|}{|\obset(\varphi')|},
\]
where $\varphi'=\prog(\varphi,\{p\})$ (and when $|\obset(\varphi')|=0$, the feature defaults to $0$).
Symmetrically, the negative propositional attention feature $f_{att}^-$ counts the appearances of $\neg q$.

In the embedding vector from Figure \ref{fig:overview}, the two teal numbers correspond to the feature values $f_{att}^+(\varphi,r,r)$ and $f_{att}^+(\varphi,r,y)$, where $\varphi =\ltlF r \land \ltlFG y$.
The first value being $0.0$ indicates that after seeing $\{r\}$ once, it is completely irrelevant.
Further, the second value being $1.0$ indicates that after seeing $\{r\}$, seeing letters containing $y$ is highly promising.
Crucially, this information is already present in the embedding of $q_0$, allowing the agent to consider its future obligation of $\ltlFG y$ while still primarily focussing on achieving $\ltlF r$.

\subsection{Policy Architecture}
Our approach is realised in a deep RL setting.
We next describe our policy architecture.
Our policies take the current MDP state $s$ and the semantic embedding of the current LDBA state $q$ as input, and produce an action distribution over the MDP action space $\mathcal A$ and possible $\varepsilon$-actions.
Recall that $\varepsilon$-actions $\varepsilon_{q'}$ (where $q'\in\mathcal E(q)$) correspond to following an $\varepsilon$-transition to $q'$ while remaining in the same MDP state.

To handle this complex hybrid action space (note in particular that $\mathcal A$ may be continuous, and $|\mathcal E|$ is not known apriori), we propose a multi-headed neural network architecture for the policy.
Given the current MDP state $s$ and LDBA state $q$, we first compute latent representations of all LDBA states $u$ in the $\varepsilon$-closure $\mathcal E(q)\cup\{q\}$ as follows: first, the MDP state $s$ is encoded via a \textit{state encoder}, which is either a multilayer perceptron (MLP) or a convolutional neural network (CNN), depending on the domain.
Then, the semantic embedding of each LDBA state $u\in\mathcal E(q)\cup\{q\}$ is processed via a learnable linear projection layer.
These feature vectors are concatenated and passed through a shared representation MLP for each LDBA state, producing latent embeddings $z_u \in \mathbb R^d$.

The latent embeddings are then fed into a linear \textit{scoring} head $\lambda \colon\mathbb R^d\to\mathbb R$ that assigns logits to $\varepsilon$-actions, and the embedding of the current LDBA state $q$ is additionally processed by an \textit{environment actor} head $\pi_{\text{env}}$ that outputs the parameters of a distribution over the MDP action space $\mathcal A$ (e.g.\ the mean and standard deviation of a Gaussian distribution).
Intuitively, the scores $\lambda(z_u)$ model the probability of taking the $\varepsilon$-action that transitions to $u$, and $\lambda(z_q)$ corresponds to the probability of executing an MDP action rather than an $\varepsilon$-action.
Formally, the policy induces the hybrid action distribution
\begin{equation*}
    \pi(a\given s,q) = \begin{cases}
                            \hat \lambda(z_{u}) & \text{if } a = \varepsilon_{u}, u\in\mathcal E(q)\\
                            \hat \lambda(z_{q})\cdot \pi_{\text{env}}(a\given z_{q}) & \text{if } a\in\mathcal A,
                        \end{cases}
\end{equation*}
where $\varepsilon_u$ is the $\varepsilon$-action that transitions to $u$ and
\begin{equation*}
    \hat \lambda(z_v) = \frac{\exp(\lambda(z_v))}{\sum_{v'\in\mathcal E(q)\cup \{q\}}\exp(\lambda(z_{v'}))}
\end{equation*}
denotes the softmax-normalised score function.
Note that if no $\varepsilon$-actions are available (i.e.\ $\mathcal E(q) = \emptyset$), then the action distribution reduces to simply $\pi_{\text{env}}(a\given z_q)$. 

\paragraph{Training procedure.}
We train the policy end-to-end via goal-conditioned RL~\cite{DBLP:conf/ijcai/LiuZ022}: at the beginning of each episode, we sample a random LTL specification $\varphi\sim\mathcal D$ and assign a reward of 1 to visits to accepting states.
To make training converge faster in practice, we follow the approach of \citeauthor{DBLP:conf/iclr/JackermeierA25}~(\citeyear{DBLP:conf/iclr/JackermeierA25}) and design a training curriculum consisting of different stages of increasingly challenging formulae.
For example, we generally first train the agent with simple formulae of the form $\ltlF a$ before sampling more complex tasks.
Once the agent performs sufficiently well on tasks sampled from the current stage, we move on to the next stage.
Our approach is independent of the underlying RL algorithm.
Throughout our experiments, we use \textit{proximal policy optimisation} (PPO)~\cite{DBLP:journals/corr/SchulmanWDRK17}.
We provide further details on training, including on curricula, in Appendix \ref{app:training}.

\section{Experiments}
We evaluate our approach, named \textsc{SemLTL}, across a variety of domains and specifications.\footnote{Code available at: \href{https://github.com/mathiasj33/semltl}{github.com/mathiasj33/semltl}}
We aim to answer the following research questions:
\textbf{(RQ1)} Are our embeddings expressive enough to overcome the myopia of decomposition-based methods?
\textbf{(RQ2)}
Can \textsc{SemLTL} successfully learn generalist policies for zero-shot executing arbitrary LTL instructions?
\textbf{(RQ3)} Does our approach successfully scale to complex specifications, where previous methods fail?%

\paragraph{RQ1.}
Consider the grid-world environment depicted in Figure~\ref{fig:rq1} (left), inspired by \cite{DBLP:conf/icml/VaezipoorLIM21}.
The agent, located on the left, receives one of the following two goals:
$\ltlF (\text{parcel}\land\ltlF\text{hammer})$ or $\ltlF (\text{parcel}\land\ltlF\text{wrench})$, requiring it to enter the correct room via the one-way conveyor belts.
For decomposition based approaches, the first subtask is $\ltlF \text{parcel}$ in both cases.
The agent thus has to choose a room (in order to pick up a parcel) without knowing the entire specification.
Since both specifications are equally likely, but the agent has no way of differentiating between them, it can achieve a maximum success rate of 50\%.
In contrast, we observe that \textsc{SemLTL} does not suffer from such myopia and quickly converges to the optimal policy (Figure \ref{fig:rq1}, right).

\begin{figure}
    \hspace{0.6cm}
    \includegraphics[width=0.9\columnwidth]{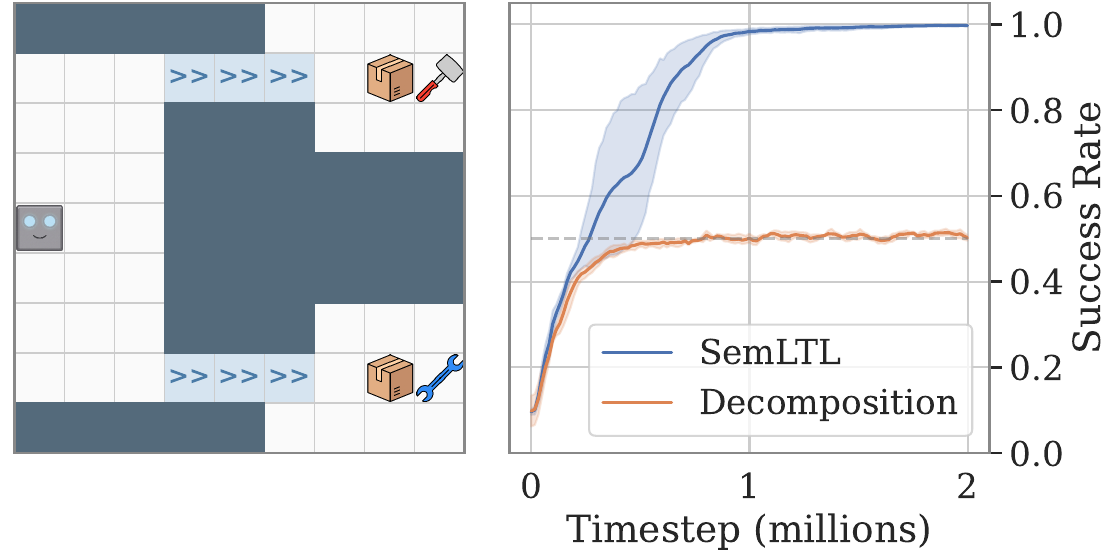}
    \caption{Conveyor environment example.
    The task is either $\ltlF (\text{parcel}\land\ltlF\text{hammer})$ or $\ltlF (\text{parcel}\land\ltlF\text{wrench})$.
    Once the agent is on a conveyor belt, it cannot go back.
    Without our propositional attention feature we can only achieve 50\% (like other decomposition based approaches), whereas our full approach finds the optimal solution.
    }
    \label{fig:rq1}
\end{figure}

\subsubsection{RQ2 and RQ3}
We now proceed to conduct a large-scale experimental study to investigate RQ2 and RQ3.

\paragraph{Environments.}
We rely on established benchmarks to evaluate \textsc{SemLTL}:
the \textit{LetterWorld} environment~\cite{DBLP:conf/icml/VaezipoorLIM21}, which consists of a $7\times 7$ grid with randomly placed letters corresponding to atomic propositions, and \textit{ZoneEnv}~\cite{DBLP:conf/icml/VaezipoorLIM21}, a high-dimensional robotic navigation environment with a continuous action space, where randomly placed zones of different colours form the atomic propositions.
We consider a more difficult variant of ZoneEnv with 8 colours (instead of 4) and a more complex observation space based on an RGB lidar sensor. 
For further details, see Appendix~\ref{app:envs}.

\begin{table*}[t]
\centering
\label{tab:adapted_results}
\resizebox{0.85\textwidth}{!}{ 
\begin{tabular}{lll rr r r rr}
\toprule
\multirow{2}{*}{} & \multirow{2}{*}{} & \multirow{2}{*}{Task Specification Type} & \multirow{2}{*}{$|\mathcal Q|$} & \multirow{2}{*}{$|\delta|$} & \textsc{LTL2Action} & \textsc{DeepLTL} & \multicolumn{2}{c}{\textsc{\textbf{SemLTL}} \textbf{(Ours)}} \\
\cmidrule(lr){6-6} \cmidrule(lr){7-7} \cmidrule(lr){8-9}
& & & & & SR / $\mu_\text{acc}$ & SR / $\mu_\text{acc}$ & SR / $\mu_\text{acc}$ & $\mu_\text{states}$ \\
\midrule

\multirow{8}{*}{\rotatebox{90}{Finite-horizon}} 
& \multirow{4}{*}{\rotatebox{90}{Letter}}
& \textsc{Small} & $9$ & $23$
& $0.67 \scriptstyle{\pm 0.12}$ 
& $\mathbf{0.98} \scriptstyle{\pm 0.01}$ 
& $\mathbf{0.98} \scriptstyle{\pm 0.01}$ & $4.12 \scriptstyle{\pm 0.02}$ \\

& & \textsc{Local-Safety}[$3,3$] & $233$ & $1,372$
& $0.57 \scriptstyle{\pm 0.07}$ 
& \timeout
& $\mathbf{0.99} \scriptstyle{\pm 0.00}$ & $5.38 \scriptstyle{\pm 0.00}$ \\

& & \textsc{Global-Safety}[$4,6$] & $1,072$ & $6,414$ 
& $0.12 \scriptstyle{\pm 0.03}$ 
& \timeout
& $\mathbf{0.91} \scriptstyle{\pm 0.01}$ & $6.77 \scriptstyle{\pm 0.04}$ \\

& & \textsc{Finite-Reactive}[$8,2$] & $300$ & $2,115$
& $0.74 \scriptstyle{\pm 0.04}$ 
& \timeout
& $\mathbf{1.00} \scriptstyle{\pm 0.00}$ & $2.80 \scriptstyle{\pm 0.10}$ \\

\cmidrule{2-9}

& \multirow{4}{*}{\rotatebox{90}{Zones}}
& \textsc{Small} & $9$ & $23$
& $0.52 \scriptstyle{\pm 0.07}$ 
& $\mathbf{0.93} \scriptstyle{\pm 0.02}$ 
& $0.92 \scriptstyle{\pm 0.02}$ & $3.40 \scriptstyle{\pm 0.01}$ \\

& & \textsc{Local-Safety}[$3,3$] & $83$ & $379$
& $0.66 \scriptstyle{\pm 0.11}$ 
& $\mathbf{0.98}\scriptstyle{\pm 0.01}$
& $0.93 \scriptstyle{\pm 0.06}$ & $5.09 \scriptstyle{\pm 0.14}$ \\

& & \textsc{Global-Safety}[$4,8$] & $1,162$ & $6,510$ 
& $0.39 \scriptstyle{\pm 0.05}$
& \timeout
& $\mathbf{0.79} \scriptstyle{\pm 0.06}$ & $6.73 \scriptstyle{\pm 0.18}$ \\

& & \textsc{Finite-Reactive}[$8,2$] & $107$ & $619$
& $0.74 \scriptstyle{\pm 0.11}$ 
& \timeout
& $\mathbf{0.98} \scriptstyle{\pm 0.00}$ & $2.96 \scriptstyle{\pm 0.16}$ \\

\midrule
\addlinespace[-0.02ex]
\midrule

\multirow{7}{*}{\rotatebox{90}{Infinite-horizon}} 
& \multirow{3}{*}{\rotatebox{90}{Letter}}
& \textsc{Small} & $7$ & $16$
& n/a 
& $6.15 \scriptstyle{\pm 1.97}$
& $\mathbf{7.13} \scriptstyle{\pm 0.51}$ & $4.76 \scriptstyle{\pm 0.09}$ \\

& & \textsc{Complex-Patrol}[$5,5$] & $49$ & $227$
& n/a 
& $1.78 \scriptstyle{\pm 0.26}$ 
& $\mathbf{1.83} \scriptstyle{\pm 0.10}$ & $6.44 \scriptstyle{\pm 0.03}$  \\

& & \textsc{Always-Reactive}[$5,1$] & $38$ & $204$
& n/a 
& \timeout
& $\mathbf{3.06} \scriptstyle{\pm 0.07}$ & $7.06 \scriptstyle{\pm 0.02}$ \\
\cmidrule{2-9}

& \multirow{4}{*}{\rotatebox{90}{Zones}}
& \textsc{Small} & $4$ & $10$
& n/a 
& $14.50 \scriptstyle{\pm 6.71}$ 
& $\mathbf{18.65} \scriptstyle{\pm 6.02}$ & $3.59 \scriptstyle{\pm 0.04}$ \\

& & \textsc{Reach-Stay}[$5$] & $34$ & $116$
& n/a 
& $914 \scriptstyle{\pm 862}$
& $\mathbf{2,503} \scriptstyle{\pm 497}$ & $7.29 \scriptstyle{\pm 0.15}$ \\

& & \textsc{Complex-Patrol}[$5,5$] & $45$ & $201$
& n/a 
& $2.13 \scriptstyle{\pm 1.45}$ 
& $\mathbf{2.44} \scriptstyle{\pm 0.44}$ & $6.33 \scriptstyle{\pm 0.10}$ \\

& & \textsc{Always-Reactive}[$5,1$] & $38$ & $204$
& n/a 
& \timeout
& $\mathbf{3.73} \scriptstyle{\pm 0.80}$ & $7.00 \scriptstyle{\pm 0.00}$ \\

\bottomrule
\end{tabular}
}

\caption{Evaluation results of \textsc{SemLTL} and the baselines.
For each task family, $|\mathcal Q|$ denotes the average number of LDBA states, and $|\delta|$ the average number of transitions.
We report success rate (SR) for finite tasks and number of visits to accepting states ($\mu_\text{acc}$) for infinite tasks, as well as number of constructed LDBA states ($\mu_\text{states}$) for \textsc{SemLTL}\@.
``\timeout''~denotes failure to produce a single action within a time limit of 600s.
Results are averaged over 5 seeds and 500 episodes per seed.
\textsc{SemLTL} demonstrates strong performance and robust scalability across tasks.}
\label{tab:main-results}
\end{table*}

\paragraph{Baselines.}
We compare our approach to \textsc{LTL2Action}, the only other approach that does not have to compute the entire automaton upfront, and \textsc{DeepLTL}, the state-of-the-art, automata-conditioned approach that does not suffer from myopia and is applicable to full LTL\@.
See Appendix~\ref{app:experiments} for implementation details and hyperparameters of \textsc{SemLTL} and the baselines, including training curricula.

\paragraph{Tasks.}
We consider a wide range of tasks for our evaluation.
First, we compare to a set of established LTL tasks from the literature~\cite{DBLP:conf/iclr/JackermeierA25}, which we aggregate and refer to as \textsc{Small} in our results.
Since these tasks are of limited complexity, we furthermore introduce several parameterised \textit{families} of both finite- and infinite-horizon tasks that enable us to systematically assess the scalability of our approach.
The \textsc{Local-Safety}[$k,m$] family (``avoidance" in \cite{DBLP:conf/icml/VaezipoorLIM21}), consists of $m$ disjunctions of reach-avoid specifications $\varphi[k]$ where $ \varphi[k] \equiv \neg a_1 \mathbin{\ltlU} (r_1 \land (\neg a_2 \mathbin{\ltlU} (r_2 \land \dots \land (\neg a_k \mathbin{\ltlU} r_k))))$.
Further, we define the variant \textsc{Global-Safety}[$k,m$], which fixes a single global proposition $a$ to avoid throughout the entire task.
We also introduce the \textsc{Finite-Reactive}[$k,m$] family, where a simple finite-goal needs to be reached.
However, whenever a trigger variable $t$ is observed, the agent needs to react by achieving specified response variables.
The overall task consists of $k$ such reaction rules, where each permits $m$ response variables. 
An example is $(\left(t\rightarrow \ltlF (r_{1}\lor r_2)\right) \ltlU g$.

For infinite-horizon tasks, we first consider the \textsc{Complex-Patrol}[$k,m$] family.
This consists of a disjunction of sequential reach tasks $\psi[k] \equiv \ltlF(r_1 \land \ltlF(\ldots \land \ltlF r_k))$ that need to be satisfied recurrently while avoiding a global safety constraint.
We also investigate the \textsc{Reach-Stay}[$k$] family of the form $\psi[k] \land \ltlFG a$, where $a$ is an eventual safety target.
Note that we do not consider this task family in LetterWorld, since the agent has to move at every timestep, and hence cannot stay continuously at the same proposition.
These two families cover complex tasks of the persistence and recurrence fragment, respectively.
We furthermore introduce the \textsc{Always-Reactive}[$k,m$] family, where tasks require maintaining reactive responses over an infinite horizon.
Concretely, a trigger $t_0$ must be seen infinitely often ($\ltlGF t_0$), while responding to a set of $k$ global reaction rules with $m$ response variables each (e.g.\ $\ltlG(t_0 \rightarrow \ltlF(r_1 \lor t_1))$, where in particular $t_1$ could be another trigger).
We choose suitable parameters to cover a broad spectrum of automaton sizes, and randomly sample 5 tasks per family for our experiments.

\paragraph{Results.}
Table \ref{tab:main-results} lists evaluation results of \textsc{SemLTL} and the baselines on both evaluation environments.
As is standard, we report success rates (SR) for finite-horizon tasks and average visits to accepting states ($\mu_{\text{acc}}$) for infinite-horizon tasks.
Intuitively, $\mu_{\text{acc}}$ captures how often the policy successfully completes an accepting cycle in the LDBA\@.
We also report the number of LDBA states constructed by \textsc{SemLTL} ($\mu_\text{states}$).

We observe that \textsc{SemLTL} in both environments successfully learns a policy that is able to generalise to diverse LTL specifications, achieving strong results throughout (\textbf{RQ2}).
\textsc{SemLTL} consistently performs much better than \textsc{LTL2Action}, which demonstrates the structure and learnability of our semantic embeddings.
For finite tasks, the semantic labelling of LDBAs only consists of the main formula.
Both approaches thus only track one formula (using formula progression) and differ solely in the way they embed that formula.
We clearly observe that \textsc{SemLTL} is able to leverage the same information much more effectively, and over a diverse set of typical RL specifications.
\textsc{DeepLTL} also achieves strong results on tasks where constructing the entire automaton and enumerating reach-avoid sequences is feasible.
This is expected, since it exhaustively analyses the entire automaton to form a complete and optimal plan to inform its policy. 
However, for more complex specifications \textsc{DeepLTL} fails to compute even a single action within the 600s time limit (denoted by \timeout), due to the prohibitively expensive sequence enumeration.
In contrast, we observe that \textsc{SemLTL} successfully scales to highly complex formulae and achieves strong performance even on large task instances with hundreds of automaton states (\textbf{RQ3}).
Notably, \textsc{SemLTL} often only considers a small fraction of the overall LDBA (see $\mu_\text{states}$ compared to $|Q|$).
This highlights the suitability of many common RL tasks for on-the-fly automata construction, and demonstrates that the theoretical benefits of \textsc{SemLTL} lead to significant advantages in practice.
We provide full experimental results, including on more parameters for the different task families, per-specification results on \textsc{Small} tasks, and visualisations of \textsc{SemLTL} trajectories in Appendix~\ref{app:experiments}.

\section{Conclusion}
We presented \textsc{SemLTL}, a novel approach to train generalist policies for satisfying arbitrary LTL specifications.
Our method is based on recent semantic LTL-to-automata translations that annotate each automaton state with rich, semantic information.
Not only does this naturally enable support for full LTL, but it also allows automata states to be computed efficiently on the fly, without constructing the entire automaton upfront.
We introduced a structured way to embed the semantic state information, and designed a suitable policy architecture based on these embeddings.
Experimental results on a wide variety of specifications demonstrate that our approach successfully learns policies that can zero-shot generalise to novel LTL tasks, while being significantly more scalable and constructing only a fraction of the automaton.

Future work will extend our approach to real-life robotics settings with unknown labelling functions and enhance scalability via the recently proposed LTLf+/PPLTL+ temporal logics~\cite{DBLP:conf/ijcai/AminofGRV25,DBLP:conf/ijcai/Giacomo0SWY25} and exploit good-for-MDP reduction pipelines ~\cite{DBLP:conf/aaai/Weinhuber26}.

\clearpage
\section*{Acknowledgements}
This work was supported in parts by the UKRI Erlangen AI Hub
on Mathematical and Computational Foundations of AI (No.\ EP/Y028872/1).
This work is part of the Intelligence-Oriented Verification\&Controller Synthesis (InOVationCS) project funded by the European Union under Grant Agreement No. 101171844.
MJ is funded by the EPSRC Centre for Doctoral Training in Autonomous Intelligent Machines and Systems (EP/S024050/1).

\bibliographystyle{named}
\bibliography{ijcai26}

@article{DBLP:journals/jmlr/NarvekarPLSTS20,
  author       = {Sanmit Narvekar and
                  Bei Peng and
                  Matteo Leonetti and
                  Jivko Sinapov and
                  Matthew E. Taylor and
                  Peter Stone},
  title        = {Curriculum Learning for Reinforcement Learning Domains: {A} Framework
                  and Survey},
  journal      = {J. Mach. Learn. Res.},
  volume       = {21},
  pages        = {181:1--181:50},
  year         = {2020}
}

@inproceedings{DBLP:conf/nips/HasseltGHMS16,
  author       = {Hado van Hasselt and
                  Arthur Guez and
                  Matteo Hessel and
                  Volodymyr Mnih and
                  David Silver},
  title        = {Learning values across many orders of magnitude},
  booktitle    = {{NIPS}},
  pages        = {4287--4295},
  year         = {2016}
}

@inproceedings{DBLP:journals/corr/KingmaB14,
  author       = {Diederik P. Kingma and
                  Jimmy Ba},
  title        = {Adam: {A} Method for Stochastic Optimization},
  booktitle    = {{ICLR} (Poster)},
  year         = {2015}
}

@inproceedings{DBLP:conf/emnlp/ChoMGBBSB14,
  author       = {Kyunghyun Cho and
                  Bart van Merrienboer and
                  {\c{C}}aglar G{\"{u}}l{\c{c}}ehre and
                  Dzmitry Bahdanau and
                  Fethi Bougares and
                  Holger Schwenk and
                  Yoshua Bengio},
  title        = {Learning Phrase Representations using {RNN} Encoder-Decoder for Statistical
                  Machine Translation},
  booktitle    = {{EMNLP}},
  pages        = {1724--1734},
  publisher    = {{ACL}},
  year         = {2014}
}

@inproceedings{DBLP:conf/nips/ZaheerKRPSS17,
  author       = {Manzil Zaheer and
                  Satwik Kottur and
                  Siamak Ravanbakhsh and
                  Barnab{\'{a}}s P{\'{o}}czos and
                  Ruslan Salakhutdinov and
                  Alexander J. Smola},
  title        = {Deep Sets},
  booktitle    = {{NIPS}},
  pages        = {3391--3401},
  year         = {2017}
}

@inproceedings{DBLP:conf/esws/SchlichtkrullKB18,
  author       = {Michael Sejr Schlichtkrull and
                  Thomas N. Kipf and
                  Peter Bloem and
                  Rianne van den Berg and
                  Ivan Titov and
                  Max Welling},
  title        = {Modeling Relational Data with Graph Convolutional Networks},
  booktitle    = {{ESWC}},
  series       = {Lecture Notes in Computer Science},
  volume       = {10843},
  pages        = {593--607},
  publisher    = {Springer},
  year         = {2018}
}

@inproceedings{ltlgnn,
    author    = {Mattia Giuri and Mathias Jackermeier and Alessandro Abate},
    title     = {Zero-Shot Instruction Following in RL via Structured {LTL} Representations},
    booktitle = {{ICML} Workshop on Programmatic Representations for Agent Learning},
    year      = {2025}
}

@article{DBLP:journals/tnn/ScarselliGTHM09,
  author       = {Franco Scarselli and
                  Marco Gori and
                  Ah Chung Tsoi and
                  Markus Hagenbuchner and
                  Gabriele Monfardini},
  title        = {The Graph Neural Network Model},
  journal      = {{IEEE} Trans. Neural Networks},
  volume       = {20},
  number       = {1},
  pages        = {61--80},
  year         = {2009}
}

@inproceedings{DBLP:journals/corr/abs-2508-01561,
  author       = {Zijian Guo and
                  Ilker Isik and
                  H. M. Sabbir Ahmad and
                  Wenchao Li},
  title        = {One Subgoal at a Time: Zero-Shot Generalization to Arbitrary Linear
                  Temporal Logic Requirements in Multi-Task Reinforcement Learning},
  booktitle      = {NeurIPS},
  year         = {2025}
}

@inproceedings{DBLP:conf/icra/LiuSRJ0T24,
  author       = {Jason Xinyu Liu and
                  Ankit Shah and
                  Eric Rosen and
                  Mingxi Jia and
                  George Konidaris and
                  Stefanie Tellex},
  title        = {Skill Transfer for Temporal Task Specification},
  booktitle    = {{ICRA}},
  pages        = {2535--2541},
  publisher    = {{IEEE}},
  year         = {2024}
}

@inproceedings{DBLP:conf/cdc/SadighKCSS14,
  author    = {Dorsa Sadigh and
               Eric S. Kim and
               Samuel Coogan and
               S. Shankar Sastry and
               Sanjit A. Seshia},
  title     = {A learning based approach to control synthesis of Markov decision
               processes for linear temporal logic specifications},
  booktitle = {{CDC}},
  pages     = {1091--1096},
  publisher = {{IEEE}},
  year      = {2014}
}

@article{DBLP:journals/ml/XuF24,
  author       = {Duo Xu and
                  Faramarz Fekri},
  title        = {Generalization of temporal logic tasks via future dependent options},
  journal      = {Mach. Learn.},
  volume       = {113},
  number       = {10},
  pages        = {7509--7540},
  year         = {2024}
}

@article{DBLP:journals/ai/SuttonPS99,
  author       = {Richard S. Sutton and
                  Doina Precup and
                  Satinder Singh},
  title        = {Between MDPs and Semi-MDPs: {A} Framework for Temporal Abstraction
                  in Reinforcement Learning},
  journal      = {Artif. Intell.},
  volume       = {112},
  number       = {1-2},
  pages        = {181--211},
  year         = {1999}
}

@inproceedings{DBLP:conf/atva/BrazdilCCFKKPU14,
  author       = {Tom{\'{a}}s Br{\'{a}}zdil and
                  Krishnendu Chatterjee and
                  Martin Chmelik and
                  Vojtech Forejt and
                  Jan K{\v r}et{\'{\i}}nsk{\'{y}} and
                  Marta Z. Kwiatkowska and
                  David Parker and
                  Mateusz Ujma},
  title        = {Verification of Markov Decision Processes Using Learning Algorithms},
  booktitle    = {{ATVA}},
  series       = {Lecture Notes in Computer Science},
  volume       = {8837},
  pages        = {98--114},
  publisher    = {Springer},
  year         = {2014}
}

@inproceedings{DBLP:conf/qest/HasanbeigKA22,
  author       = {Mohammadhosein Hasanbeig and
                  Daniel Kroening and
                  Alessandro Abate},
  title        = {{LCRL:} Certified Policy Synthesis via Logically-Constrained Reinforcement
                  Learning},
  booktitle    = {{QEST}},
  series       = {Lecture Notes in Computer Science},
  volume       = {13479},
  pages        = {217--231},
  publisher    = {Springer},
  year         = {2022}
}

@inproceedings{DBLP:conf/icra/Bozkurt0ZP20,
  author       = {Alper Kamil Bozkurt and
                  Yu Wang and
                  Michael M. Zavlanos and
                  Miroslav Pajic},
  title        = {Control Synthesis from Linear Temporal Logic Specifications using
                  Model-Free Reinforcement Learning},
  booktitle    = {{ICRA}},
  pages        = {10349--10355},
  publisher    = {{IEEE}},
  year         = {2020}
}

@inproceedings{DBLP:conf/icml/Voloshin0Y23,
  author       = {Cameron Voloshin and
                  Abhinav Verma and
                  Yisong Yue},
  title        = {Eventual Discounting Temporal Logic Counterfactual Experience Replay},
  booktitle    = {{ICML}},
  series       = {Proceedings of Machine Learning Research},
  volume       = {202},
  pages        = {35137--35150},
  publisher    = {{PMLR}},
  year         = {2023}
}

@inproceedings{DBLP:conf/tacas/HahnPSSTW19,
  author    = {Ernst Moritz Hahn and
               Mateo Perez and
               Sven Schewe and
               Fabio Somenzi and
               Ashutosh Trivedi and
               Dominik Wojtczak},
  title     = {Omega-Regular Objectives in Model-Free Reinforcement Learning},
  booktitle = {{TACAS} {(1)}},
  series    = {Lecture Notes in Computer Science},
  volume    = {11427},
  pages     = {395--412},
  publisher = {Springer},
  year      = {2019}
}

@inproceedings{DBLP:conf/iclr/JackermeierA25,
  author    = {Mathias Jackermeier and
               Alessandro Abate},
  title     = {DeepLTL: Learning to Efficiently Satisfy Complex {LTL} Specifications
               for Multi-Task {RL}},
  booktitle = {{ICLR}},
  publisher = {OpenReview.net},
  year      = {2025}
}

@inproceedings{DBLP:conf/icml/VaezipoorLIM21,
  author    = {Pashootan Vaezipoor and
               Andrew C. Li and
               Rodrigo Toro Icarte and
               Sheila A. McIlraith},
  title     = {LTL2Action: Generalizing {LTL} Instructions for Multi-Task {RL}},
  booktitle = {{ICML}},
  series    = {Proceedings of Machine Learning Research},
  volume    = {139},
  pages     = {10497--10508},
  publisher = {{PMLR}},
  year      = {2021}
}

@article{DBLP:journals/jacm/EsparzaKS20,
  author  = {Javier Esparza and
             Jan K{\v r}et{\'{\i}}nsk{\'{y}} and
             Salomon Sickert},
  title   = {A Unified Translation of Linear Temporal Logic to {\(\omega\)}-Automata},
  journal = {J. {ACM}},
  volume  = {67},
  number  = {6},
  pages   = {33:1--33:61},
  year    = {2020}
}

@inproceedings{DBLP:conf/tacas/KretinskyMPZ25,
  author    = {Jan K{\v r}et{\'{\i}}nsk{\'{y}} and
               Tobias Meggendorfer and
               Maximilian Prokop and
               Ashkan Zarkhah},
  title     = {SemML: Enhancing Automata-Theoretic {LTL} Synthesis with Machine Learning},
  booktitle = {{TACAS} {(1)}},
  series    = {Lecture Notes in Computer Science},
  volume    = {15696},
  pages     = {233--253},
  publisher = {Springer},
  year      = {2025}
}

@inproceedings{DBLP:conf/atva/KretinskyMM19,
  author    = {Jan K{\v r}et{\'{\i}}nsk{\'{y}} and
               Alexander Manta and
               Tobias Meggendorfer},
  title     = {Semantic Labelling and Learning for Parity Game Solving in {LTL} Synthesis},
  booktitle = {{ATVA}},
  series    = {Lecture Notes in Computer Science},
  volume    = {11781},
  pages     = {404--422},
  publisher = {Springer},
  year      = {2019}
}

@inproceedings{DBLP:conf/cav/KretinskyMPR23,
  author    = {Jan K{\v r}et{\'{\i}}nsk{\'{y}} and
               Tobias Meggendorfer and
               Maximilian Prokop and
               Sabine Rieder},
  title     = {Guessing Winning Policies in {LTL} Synthesis by Semantic Learning},
  booktitle = {{CAV} {(1)}},
  series    = {Lecture Notes in Computer Science},
  volume    = {13964},
  pages     = {390--414},
  publisher = {Springer},
  year      = {2023}
}

@inproceedings{DBLP:conf/ijcai/LiuZ022,
  author    = {Minghuan Liu and
               Menghui Zhu and
               Weinan Zhang},
  title     = {Goal-Conditioned Reinforcement Learning: Problems and Solutions},
  booktitle = {{IJCAI}},
  pages     = {5502--5511},
  publisher = {ijcai.org},
  year      = {2022}
}

@article{DBLP:journals/corr/SchulmanWDRK17,
  author  = {John Schulman and
             Filip Wolski and
             Prafulla Dhariwal and
             Alec Radford and
             Oleg Klimov},
  title   = {Proximal Policy Optimization Algorithms},
  journal = {CoRR},
  volume  = {abs/1707.06347},
  year    = {2017}
}

@inproceedings{DBLP:conf/iclr/LeonSB22,
  author       = {Borja G. Le{\'{o}}n and
                  Murray Shanahan and
                  Francesco Belardinelli},
  title        = {In a Nutshell, the Human Asked for This: Latent Goals for Following
                  Temporal Specifications},
  booktitle    = {{ICLR}},
  publisher    = {OpenReview.net},
  year         = {2022}
}

@inproceedings{DBLP:conf/nips/0002M023,
  author       = {Wenjie Qiu and
                  Wensen Mao and
                  He Zhu},
  title        = {Instructing Goal-Conditioned Reinforcement Learning Agents with Temporal
                  Logic Objectives},
  booktitle    = {NeurIPS},
  year         = {2023}
}

@inproceedings{DBLP:conf/cav/SickertEJK16,
  author       = {Salomon Sickert and
                  Javier Esparza and
                  Stefan Jaax and
                  Jan K{\v r}et{\'{\i}}nsk{\'{y}}},
  title        = {Limit-Deterministic B{\"{u}}chi Automata for Linear Temporal
                  Logic},
  booktitle    = {{CAV} {(2)}},
  series       = {Lecture Notes in Computer Science},
  volume       = {9780},
  pages        = {312--332},
  publisher    = {Springer},
  year         = {2016}
}

@inproceedings{DBLP:conf/focs/Pnueli77,
  author       = {Amir Pnueli},
  title        = {The Temporal Logic of Programs},
  booktitle    = {{FOCS}},
  pages        = {46--57},
  publisher    = {{IEEE} Computer Society},
  year         = {1977}
}

@book{DBLP:books/lib/SuttonB2018,
  author       = {Richard S. Sutton and
                  Andrew G. Barto},
  title        = {Reinforcement learning - an introduction, 2nd Edition},
  publisher    = {{MIT} Press},
  year         = {2018}
}

@inproceedings{DBLP:conf/aaai/BacchusK96,
  author       = {Fahiem Bacchus and
                  Froduald Kabanza},
  title        = {Planning for Temporally Extended Goals},
  booktitle    = {AAAI/IAAI, Vol. 2},
  pages        = {1215--1222},
  publisher    = {{AAAI} Press / The {MIT} Press},
  year         = {1996}
}

@inproceedings{DBLP:conf/aaai/Weinhuber26,
  author       = {Christoph Weinhuber and
                  Giuseppe {De Giacomo} and 
                  Yong Li and
                  Sven Schewe and
                  Qiyi Tang},
  title        = {Good-for-MDP State Reduction for Stochastic LTL Planning},
  booktitle    = {{AAAI}},
  publisher    = {{AAAI} Press},
  year         = {2026}
}

@inproceedings{DBLP:conf/cav/KretinskyE12,
  author       = {Jan K{\v r}et{\'{\i}}nsk{\'{y}} and
                  Javier Esparza},
  editor       = {P. Madhusudan and
                  Sanjit A. Seshia},
  title        = {Deterministic Automata for the (F, G)-Fragment of {LTL}},
  booktitle    = {Computer Aided Verification - 24th International Conference, {CAV}
                  2012, Berkeley, CA, USA, July 7-13, 2012 Proceedings},
  series       = {Lecture Notes in Computer Science},
  volume       = {7358},
  pages        = {7--22},
  publisher    = {Springer},
  year         = {2012},
  url          = {https://doi.org/10.1007/978-3-642-31424-7\_7},
  doi          = {10.1007/978-3-642-31424-7\_7},
  bibsource    = {dblp computer science bibliography, https://dblp.org}
}

@inproceedings{DBLP:conf/cav/EsparzaK14,
  author       = {Javier Esparza and
                  Jan K{\v r}et{\'{\i}}nsk{\'{y}}},
  editor       = {Armin Biere and
                  Roderick Bloem},
  title        = {From {LTL} to Deterministic Automata: {A} Safraless Compositional
                  Approach},
  booktitle    = {Computer Aided Verification - 26th International Conference, {CAV}
                  2014, Held as Part of the Vienna Summer of Logic, {VSL} 2014, Vienna,
                  Austria, July 18-22, 2014. Proceedings},
  series       = {Lecture Notes in Computer Science},
  volume       = {8559},
  pages        = {192--208},
  publisher    = {Springer},
  year         = {2014},
  url          = {https://doi.org/10.1007/978-3-319-08867-9\_13},
  doi          = {10.1007/978-3-319-08867-9\_13},
  bibsource    = {dblp computer science bibliography, https://dblp.org}
}

@inproceedings{Owl,
  author       = {Jan K{\v r}et{\'{\i}}nsk{\'{y}} and
                  Tobias Meggendorfer and
                  Salomon Sickert},
  editor       = {Shuvendu K. Lahiri and
                  Chao Wang},
  title        = {Owl: {A} Library for {\(\omega\)}-Words, Automata, and {LTL}},
  booktitle    = {Automated Technology for Verification and Analysis - 16th International
                  Symposium, {ATVA} 2018, Los Angeles, CA, USA, October 7-10, 2018,
                  Proceedings},
  series       = {Lecture Notes in Computer Science},
  volume       = {11138},
  pages        = {543--550},
  publisher    = {Springer},
  year         = {2018},
  url          = {https://doi.org/10.1007/978-3-030-01090-4\_34},
  doi          = {10.1007/978-3-030-01090-4\_34},
  bibsource    = {dblp computer science bibliography, https://dblp.org}
}

@inproceedings{10.1109/TIME.2013.19,
author = {Li, Jianwen and Zhang, Lijun and Pu, Geguang and Vardi, Moshe Y. and He, Jifeng},
title = {LTL Satisfiability Checking Revisited},
year = {2013},
isbn = {9781479922413},
publisher = {IEEE Computer Society},
address = {USA},
url = {https://doi.org/10.1109/TIME.2013.19},
doi = {10.1109/TIME.2013.19},
booktitle = {Proceedings of the 2013 20th International Symposium on Temporal Representation and Reasoning},
pages = {91–98},
numpages = {8},
series = {TIME '13}
}

@inproceedings{LTL_Normalization,
author = {Sickert, Salomon and Esparza, Javier},
title = {An Efficient Normalisation Procedure for Linear Temporal Logic and Very Weak Alternating Automata},
year = {2020},
isbn = {9781450371049},
publisher = {Association for Computing Machinery},
address = {New York, NY, USA},
url = {https://doi.org/10.1145/3373718.3394743},
doi = {10.1145/3373718.3394743},
booktitle = {Proceedings of the 35th Annual ACM/IEEE Symposium on Logic in Computer Science},
pages = {831–844},
numpages = {14},
location = {Saarbr\"{u}cken, Germany},
series = {LICS '20}
}

@inproceedings{impossibility_result,
author = {Hahn, Ernst Moritz and Perez, Mateo and Schewe, Sven and Somenzi, Fabio and Trivedi, Ashutosh and Wojtczak, Dominik},
title = {An Impossibility Result in Automata-Theoretic Reinforcement Learning},
year = {2022},
isbn = {978-3-031-19991-2},
publisher = {Springer-Verlag},
address = {Berlin, Heidelberg},
url = {https://doi.org/10.1007/978-3-031-19992-9_3},
doi = {10.1007/978-3-031-19992-9_3},
booktitle = {Automated Technology for Verification and Analysis: 20th International Symposium, ATVA 2022, Virtual Event, October 25–28, 2022, Proceedings},
pages = {42–57},
numpages = {16}
}

@article{DBLP:journals/corr/OhSLK17,
  author       = {Junhyuk Oh and
                  Satinder Singh and
                  Honglak Lee and
                  Pushmeet Kohli},
  title        = {Zero-Shot Task Generalization with Multi-Task Deep Reinforcement Learning},
  journal      = {CoRR},
  volume       = {abs/1706.05064},
  year         = {2017},
  url          = {http://arxiv.org/abs/1706.05064},
  eprinttype    = {arXiv},
  eprint       = {1706.05064},
  bibsource    = {dblp computer science bibliography, https://dblp.org}
}

@inproceedings{DBLP:conf/ijcai/Giacomo0SWY25,
  author       = {Giuseppe {De Giacomo} and
                  Yong Li and
                  Sven Schewe and
                  Christoph Weinhuber and
                  Pian Yu},
  title        = {Solving MDPs with LTLf+ and {PPLTL+} Temporal Objectives},
  booktitle    = {{IJCAI}},
  pages        = {8491--8499},
  publisher    = {ijcai.org},
  year         = {2025}
}

@inproceedings{DBLP:conf/ijcai/AminofGRV25,
  author       = {Benjamin Aminof and
                  Giuseppe {De Giacomo} and
                  Sasha Rubin and
                  Moshe Y. Vardi},
  title        = {LTLf+ and {PPLTL+:} Extending LTLf and {PPLTL} to Infinite Traces},
  booktitle    = {{IJCAI}},
  pages        = {8447--8455},
  publisher    = {ijcai.org},
  year         = {2025}
}

@misc{arxivVersion,
      title={Semantically Labelled Automata for Multi-Task Reinforcement Learning with LTL Instructions}, 
      author={Alessandro Abate and Giuseppe {De Giacomo} and Mathias Jackermeier and Jan Kretínský and Maximilian Prokop and Christoph Weinhuber},
      year={2026},
      eprint={2602.06746},
      archivePrefix={arXiv},
      primaryClass={cs.AI},
      url={https://arxiv.org/abs/2602.06746}, 
}


\appendix

\clearpage
\section{Formal definitions}\label{appendix:formal_definitions}
\subsection{ Linear Temporal Logic Satisfaction Relation}
\label{appendix:ltl_sat_relation}
Let $w$ be an infinite sequence of letters and $\varphi$ be a \ltl formulae.
The satisfaction relation $w\models\varphi$ is defined inductively, as follows:
\[
\begin{array}{ll}
w \models \ttrue                \!& \iff \text{true}, \\
w \models \ffalse               \!& \iff \text{false}, \\
w \models a                     \!& \iff a \in w[0], \\
w \models \lnot \varphi         \!& \iff w \not\models \varphi, \\
w \models \ltlG\varphi         \!& \iff \forall k.\; w_k \models \varphi, \\
w \models \ltlF\varphi         \!& \iff \exists k.\; w_k \models \varphi, \\
w \models \varphi \ltlU \psi   \!& \iff \exists k.\bigl(w_k \models \psi \land \forall\, j < k,\; w_j \models \varphi\bigr) \\
\end{array} 
\]
\subsection{Formula progression}\label{app:prog}
We present a modern version of formula progression from \cite{DBLP:journals/jacm/EsparzaKS20}.
Let $\varphi,\psi$ be \ltl formulae over \ap{} and $\sigma\in 2^{\ap}$.
The full semantics of the formula progression operator $\prog$ are inductively defined as follows:

\begin{align*}
&\prog(\ttrue,\sigma) &&= \ttrue & \\
&\prog(a,\sigma) &&= \begin{cases}
\ttrue, \text{if } a\in\sigma \\\ffalse, \text{if } a\not\in\nu
\end{cases}& \\
&\prog(\neg\varphi, \sigma) &&= \neg\prog(\varphi, \sigma) \\
&\prog(\varphi\land\psi,\sigma) &&= \phantom{\neg}\prog(\varphi,\sigma)\land \prog(\psi,\sigma) & \\
&\prog(\ltlX\varphi,\sigma) &&= \phantom{\neg}\varphi & \\
&\prog(\ltlF\varphi,\sigma) &&= \phantom{\neg}\prog(\varphi,\sigma)\lor\ltlF\varphi & \\
&\prog(\ltlG\varphi,\sigma) &&= \phantom{\neg}\prog(\varphi,\sigma)\land \ltlG\varphi & \\
&\prog(\varphi\ltlU\psi,\sigma) &&= \phantom{\neg}\prog(\psi,\sigma) \lor (\prog(\varphi,\sigma)\land (\varphi \ltlU \psi)) &
\end{align*}

\subsection{Obligation Sets}\label{app:obset}
Obligation sets are originally defined by \cite{10.1109/TIME.2013.19} to approximate infinite behavior of \ltl formulae and detect easy solutions for \ltl satisfiability.
We slightly adapt the definition s.t.\ it reasons about sets of letters $\sigma$, which allows for an efficient computation using binary decision diagrams.
Formally, for a formula $\varphi$, the obligation sets $\obset$ are defined inductively as:
\begin{align*}
&\obset(\ttrue) &&= 2^\ap & \\
&\obset(a) &&= \{\sigma\mid\sigma\in2^\ap.a \in \sigma\} & \\
&\obset(\neg\varphi) &&= 2^\ap \setminus \obset(\varphi) & \\
&\obset(\varphi\land\psi) &&= \obset(\varphi)\cap \obset(\psi) & \\
&\obset(\ltlX\varphi) &&= \obset(\varphi) & \\
&\obset(\ltlF\varphi) &&= \obset(\varphi) & \\
&\obset(\ltlG\varphi) &&= \obset(\varphi) & \\
&\obset(\varphi\ltlU\psi) &&= \obset(\psi) & \\
\end{align*}
The major property of obligation sets is that for any letter $\sigma \in \obset(\varphi)$, we have $\sigma^\omega\models\varphi$, which is given as Theorem 1 in \cite{10.1109/TIME.2013.19}.

\section{Derivation of DeepLTL Complexity}
\label{app:complexity}
\begin{theorem}
Let $\mathcal{Q}$ be the set of states in the LDBA, $\ap$ be the set of atomic propositions, and \textsf{Paths} be the set of all simple paths (i.e.\ paths without loops) starting from the current state $q\in\mathcal Q$.
Let $C(n)$ denote the complexity of evaluating the value function on a sequence of length $n$. The worst-case time complexity of computing reach-avoid sequences (Algorithm 1;~\cite{DBLP:conf/iclr/JackermeierA25}) is:
\[
\mathcal{O}\left(|\mathsf{Paths}|\cdot \left(2^{|\ap|}\cdot |\mathcal Q| + C(|\mathcal Q|)\right)\right)
\]
where $|\mathsf{Paths}|$ is upper-bounded by $e\cdot|\mathcal{Q}|!$.
\end{theorem}

\begin{proof} 
The algorithm performs a depth-first search that induces a recursion tree $\mathcal{T}$, where each node represents a unique simple path from the root. Thus, the total number of recursive calls is exactly $|\mathsf{Paths}|$.
In the worst-case scenario (a fully connected LDBA), $|\mathsf{Paths}|$ is the sum of all $k$-permutations of states for path lengths $k=1$ to $|\mathcal{Q}|$:
\[
|\mathsf{Paths}| = \sum_{k=1}^{|\mathcal{Q}|} \frac{|\mathcal{Q}|!}{(|\mathcal{Q}|-k)!} = |\mathcal{Q}|! \sum_{k=0}^{|\mathcal{Q}|-1} \frac{1}{k!} \xrightarrow{|\mathcal Q|\rightarrow \infty} e \cdot |\mathcal{Q}|!
\]

For each visited state (node in $\mathcal{T}$), the algorithm iterates through the alphabet $\Sigma = 2^{\ap} \cup \mathcal E$ and checks if the next state $q'$ exists in the current path $p$. The combined complexity of this is $\mathcal O(2^{|\ap|}\cdot|\mathcal Q|)$.
Finally, the value function $V^\pi$ is computed for a subset of extracted paths. The number of evaluations is upper-bounded by $|\mathsf{Paths}|$, yielding an asymptotic evaluation cost of $|\mathsf{Paths}|\cdot C(|\mathcal{Q}|)$.

Taking the complexities together and factoring out $|\mathsf{Paths}|$ yields the stated complexity.
\end{proof}

\section{Example with Breakpoint Formulae} \label{app:breakpoint}
We present another example automaton that uses the breakpoint formulae in the semantic labelling.
Consider the formula $\ltlF r \land \ltlGF(y \land \ltlX y)$, which denotes that we need to see $r$ \textit{once} and we need to see two $y$s in direct succession \textit{infinitely often}.
The latter being a recurrence formula means we need to track the progress of the inner formula $\ltlF(y \land \ltlX y)$ using the breakpoint formula.

The resulting LDBA with semantic labelling can be seen in Figure \ref{fig:breakpoint_aut}.
The states $q_0$ and $q_1$ work as in the prior example but don't contain breakpoint formulae just yet, as we are still in the initial component of the LDBA.
After transitioning to the accepting component, we populate the breakpoint formula in $q_2$ with the inner formula of the recurrence, namely $\ltlF(y \land \ltlX y)$.
Seeing a single $y$ transitions to state $q_3$, where the remaining obligation of the breakpoint formula is to either see another $y$ (completing the two successive $y$s) or seeing two successive $y$s some time in the future.
If we do see another $y$ we transition to $q_4$, which is accepting as the breakpoint formula has progressed to $\ttrue$.
After that we reset the breakpoint formula, which corresponds returning to $q_2$.
Thus, we see an accepting state whenever we have two $y$s in succession and thus we see an accepting state infinitely often exactly if we see two $y$s infinitely often.
 
\begin{figure}[h]
    \centering
    \begin{tikzpicture}[
        >={Stealth[round]}, 
        thick,
        font=\small,
        node distance=0.6cm, 
        box/.style={draw, rectangle, rounded corners=5pt, minimum size=1cm, fill=white},
        autstate/.style={draw, circle, inner sep=2pt, minimum size=.4cm, fill=white},
        section label/.style={font=\bfseries\footnotesize, color=gray!80, align=center},
        annotation/.style={font=\scriptsize\itshape, color=gray!90, align=left}
    ]

    \node[autstate, initial, initial text=, fill=autState_color!50] (q0) at (0,0) {$q_0$};
    \node[autstate, fill=autState_color2!50] (q1) at ($(q0) + (1.2,0)$) {$q_1$};
    
    \node[autstate, fill=autState_color3!50] (q2) at ($(q1) + (0, -1.2)$) {$q_2$};
    \node[autstate, fill=autState_color4!50] (q3) at ($(q2) - (1.2, 0)$) {$q_3$};
    \node[autstate, fill=autState_color5!50, double] (q4) at ($(q3) + (0.6, -1)$) {$q_4$};

    \path[->] 
        (q0) edge[] node[above] {$r$} (q1)
        (q1) edge[] node[right] {$\varepsilon$} (q2)
        (q2) edge[bend right=20] node[above] {$y$} (q3)
        (q3) edge[bend right=20] node[below] {$\neg y$} (q2)
        (q0) edge[loop, in=70, out=110,  looseness=5] node[left,near start] {$\neg r$} (q0)
        (q1) edge[loop, in=70, out=110,  looseness=5] node[right,near end] {$\ttrue$} (q1)
        (q2) edge[loop, in=20, out=340,  looseness=5] node[above ] {$\neg y$} (q2)
        (q3) edge[in=150,out=south] node[below left] {$y$} (q4)
        (q4) edge[in=south,out=20] node[below right] {$\ttrue$} (q2)
        ;
    

    \node[right=1cm of q1, box, inner sep=-1pt, fill=autState_color2!30] (semLab1) {
        \setlength{\tabcolsep}{3pt}
        \scriptsize
        \begin{tabular}{c|cr}
        \multirow{2}{*}{$q_1$} &
            ~M: & $\phantom{\ltlF r \land }\ltlGF (y \land \ltlX y)~$ \\
          &  ~B: & $-~$
        \end{tabular}
    };    
    \node[above=0.1cm of semLab1, box, inner sep=-1pt, fill=autState_color!30]  (semLab0) {
        \setlength{\tabcolsep}{3pt}
        \scriptsize
        \begin{tabular}{c|cr}
        \multirow{2}{*}{$q_0$} &
            ~M: &  $\ltlF r \land \ltlGF (y \land \ltlX y)~$ \\
           &  ~B: & $-~$
        \end{tabular}
    };
    \node[below=0.1 of semLab1, box, inner sep=-1pt, fill=autState_color3!30] (semLab2) {

        \setlength{\tabcolsep}{3pt}
        \scriptsize
        \begin{tabular}{c|cr}
        \multirow{2}{*}{$q_2$} &
           ~M: &  $\phantom{\ltlF r \land }\ltlGF (y \land \ltlX y)~$ \\
          &   ~B: & $\ltlF (y \land \ltlX y)~$
        \end{tabular}
    };  
    \node[below=0.1cm of semLab2, box, inner sep=-1pt, fill=autState_color4!30] (semLab3) {
        \setlength{\tabcolsep}{3pt}
        \scriptsize
        \begin{tabular}{c|cr}
        \multirow{2}{*}{$q_3$} &
            ~M: &  $\phantom{\ltlF r \land }\ltlGF (y \land \ltlX y)~$ \\
           &  ~B: & $y \lor \ltlF (y \land \ltlX y)~$
        \end{tabular}
    };
    \node[below=0.1cm of semLab3, box, inner sep=-1pt, fill=autState_color5!30] (semLab4) {
        \setlength{\tabcolsep}{3pt}
        \scriptsize
        \begin{tabular}{c|cr}
        \multirow{2}{*}{$q_4$} &
            ~M: &  $\phantom{\ltlF r \land }\ltlGF (y \land \ltlX y)~$ \\
           &  ~B: & $\ttrue~$
        \end{tabular}
    };

    \draw[->, dashed, gray!80, shorten >=2pt, shorten <=2pt] (q0.north east) to[bend left=30] ($(semLab0.west)$);
    \draw[->, dashed, gray!80, shorten >=2pt, shorten <=2pt] (q3.south east) to[bend right=15] ($(semLab3.west)$);
    \draw[->, dashed, gray!80, shorten >=2pt, shorten <=2pt] (q1.south east) to[bend right=25] ($(semLab1.west)$);
    \draw[->, dashed, gray!80, shorten >=2pt, shorten <=2pt] (q2.south east) to[bend right=15] ($(semLab2.west)$);
    \draw[->, dashed, gray!80, shorten >=2pt, shorten <=2pt] (q4.south east) to[bend right=15] ($(semLab4.west)$);
\end{tikzpicture}

    \caption{Semantically labelled automaton for the formula $\ltlF r \land \ltlGF(y \land \ltlX y)$.}
    \label{fig:breakpoint_aut}
\end{figure}

\section {List of All Features} \label{App:features}
We present variants of trueness that arise through different normalisations and the missing formula complexity features.
\subsection{Trueness variants}
We compute a trueness feature value for every MDP label of $\mdplabel$.
Normalizing in different ways and over different subsets yields different semantic meaning of the feature

\paragraph{Min-Max Normalisation.}
We compute the min and max value across all letters and normalise a feature value $x$ to the [0,1] interval using:
\[
\hat{x} = \frac{x - min}{max-min}
\]
This yields a relative ranking of letters, where higher values are preferred.
If every letter looks bad right now (i.e.\ all trueness feature values are negative) this normalisation at least allows the agent to identify which letter is least bad and attempt it which most of the time is still better than doing nothing.

\paragraph{Extreme value normalisation.}
We again compute min and max values across all letters and isolate them by setting every other value to 0.
\[
 \hat{x} = \begin{cases}
     x & \text{if}~~ x=min ~~\textbf{or}~~x=max \\
     0 & \text{else}
 \end{cases}
\]
This spoonfeeds the best and worst option to the agent.
The idea is to have a strong signal early in training that suffices for the simple formulae of early curriculum stages and later have the agent learn when to deviate and take more nuanced information into account.

\paragraph{Reach-Avoid normalisation.}
We again compute min and max values across all letters.
However, we now scale positive trueness values (letters that we want to ``reach") and negative trueness values (letters to ``avoid") separateny to the [0,1] and the [-1,0] interval respectively.
\[
 \hat{x} = \begin{cases}
     x/max & \text{if}~~ x> 0 \\
     x/min & \text{else}
 \end{cases}
\]
The idea is that the large negative value of an immediate safety violation could overshadow a desirable letter that achieves a subgoal.
While highlighting the safety property seems somewhat desirable, this is taken care of by different variants already.
In this variant, we want to highlight letters that create immegiate progress and rank them by how much progress they create.

\subsection{Formula Complexity Measures}
When the agent resolves the non-determinism of the LDBA, i.e.\ selects one of the possibly plenty-availible epsilon transitions, it requires information about the formulae of the subsequent states in the accepting component of the LDBA.
While the previously described features yield information of which letters to reach and avoid in these formulae, they yield no indication about the complexity of the remaining task.
E.g. when the agent can chose between an $\varepsilon$-action leading to $\ltlG a$ and another $\varepsilon$-action leading to $\ltlGF a \land \ltlGF b \land \ltlG \neg c$, we'd prefer to chose the former, as satisfying it is much easier (just, from looking at the task alone; the feasibility in the MDP is a different story that requires a separate judgement).

For that matter, we take inspiration from \cite{DBLP:conf/cav/KretinskyMPR23} and introduce the following measures of formula complexity.

\paragraph{Syntax tree height.}
Intuitively, a recurrence formula with a long sequence of tasks like $\ltlGF(a \land \neg b \ltlU (c \land \ldots))$ is more complex than a simple recurrence task like $\ltlGF(a) \land \ltlGF(b)$.
We detect this difference in complexity by measuring the height of their syntax trees.
To obtain a well-behaved feature value in the same order of magnitude as other features, we estimate a maximum height (the height of the initial formula) and use it to normalise all feature values.

\paragraph{Conjuncts/Disjuncts.}
Additionally, we want to measure the breadth of the syntax tree, which is mostly driven by con- and disjunctions.
We distinguish broadness due to conjunctions or disjunctions, as they indicate opposing properties.
A high number of conjunctions is restrictive (and thus likely to be more difficult to satisfy) whereas a high number of disjunctions is permissive (and thus likely to be easier to satisfy).
We employ a similar normalisation as for syntax tree height, by estimating a maximum from the initial formula.

\paragraph{Plain Trueness.}
Without computing the difference to its predecessor and without any normalisations, trueness can also act as a measure of formula complexity and formula difficulty.

\section{Training Details}\label{app:training}
In this section, we provide further technical details on our training procedure.
Instead of directly optimising the theoretical objective in Equation~\ref{eq:vanilla-automata-theory-single-rlk} by giving rewards of $1$ for visits to accepting states and $0$ otherwise, we furthermore penalise the policy with a reward of $-1$ for reaching a sink state in the automaton.
This is important in practice, as it ensures that the policy can not only learn from successful episodes, but also improve from unsuccessful trials.
We also choose a fixed discount factor $\gamma$ as is common in RL, instead of employing advanced discounting techniques that yield theoretically optimal policies, such as eventual discounting~\cite{DBLP:conf/icml/Voloshin0Y23}.
Consistent with previous work~\cite{DBLP:conf/iclr/JackermeierA25}, we find that standard discounting provides a better training signal in practice.
Finally, we terminate episodes in which a guarantee formula has been sampled early upon reaching an accepting state (since the trajectory then satisfies the formula by definition).
Since this introduces different scales of returns for different types of formulae (the agent can achieve returns greater than 1 only for non-guarantee formulae), we employ reward normalisation~\cite{DBLP:conf/nips/HasseltGHMS16} to stabilise training.

\paragraph{Curriculum learning.}
We use \textit{curriculum learning} \cite{DBLP:journals/jmlr/NarvekarPLSTS20}, which has previously been shown to improve training of LTL-conditioned multi-task RL policies~\cite{DBLP:conf/iclr/JackermeierA25}.
A training curriculum consists of different stages, where each stage induces a distribution over LTL formulae sampled during training.
Once the agent performs sufficiently well on tasks sampled from the current stage (measured by the achieved return over the last $T$ episodes), we move on to the next curriculum stage.
The purpose of this scheme is to gradually expose the agent to more difficult tasks as its capabilities improve.

\section{Experimental Details} \label{app:experiments}

\subsection{Environments}\label{app:envs}
\begin{figure}[t]
    \centering
    \begin{subfigure}{0.32\columnwidth}
        \centering
        \frame{\includegraphics[width=\textwidth]{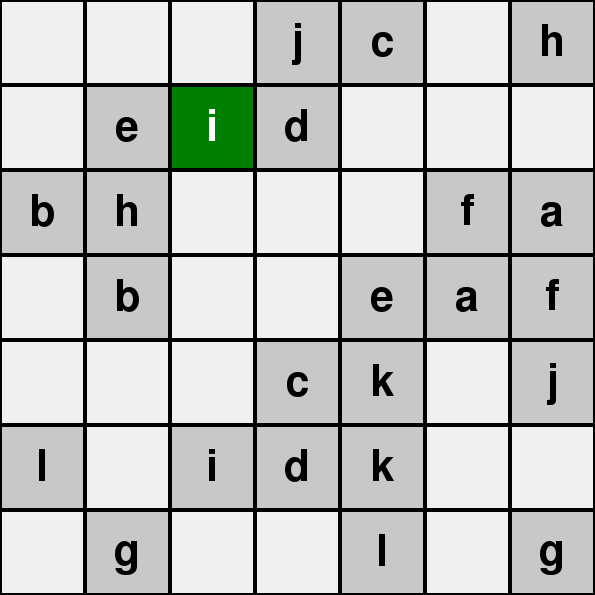}}
        \caption{LetterWorld}
        \label{fig:letter}
    \end{subfigure}
    \hspace{0.5cm}
    \begin{subfigure}{0.45\columnwidth}
        \centering
        \includegraphics[width=\textwidth]{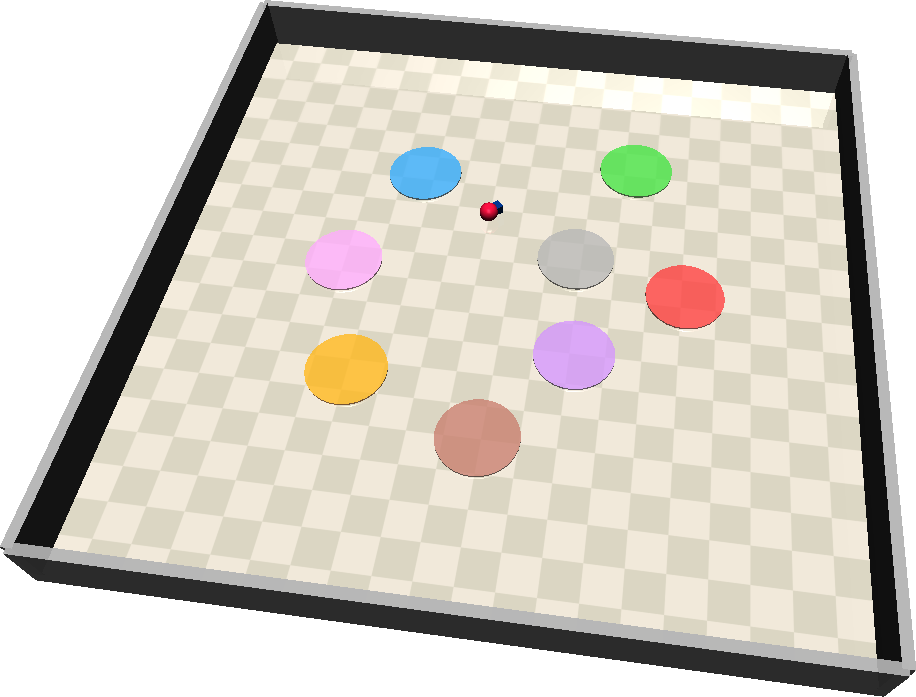}
        \vspace{-.5cm}
        \caption{ZoneEnv}
        \label{fig:zones}
    \end{subfigure}
    \caption{Environment visualisations.}
    \label{fig:envs}
\end{figure}

\begin{table*}[t]
\centering
\resizebox{0.9\textwidth}{!}{%
\begin{tabular}{llcccccc}
\toprule
\multirow{2}{*}{Category} & \multirow{2}{*}{Hyperparameter} & \multicolumn{3}{c}{LetterWorld} & \multicolumn{3}{c}{ZoneEnv} \\
\cmidrule(lr){3-5} \cmidrule(lr){6-8}
 &  & \textsc{LTL2Action} & \textsc{DeepLTL} & \textsc{SemLTL} & \textsc{LTL2Action} & \textsc{DeepLTL} & \textsc{SemLTL} \\
\midrule
\multirow{14}{*}{PPO} 
& Total timesteps & \multicolumn{3}{c}{--- $2 \times 10^7$ ---} & \multicolumn{3}{c}{--- $3 \times 10^8$ ---} \\
& Environments & \multicolumn{3}{c}{--- 16 ---} & \multicolumn{3}{c}{--- 128 ---} \\
& Steps/update & \multicolumn{3}{c}{--- 128 ---} & \multicolumn{3}{c}{--- 1024 ---} \\
& Minibatches & \multicolumn{3}{c}{--- 8 ---} & \multicolumn{3}{c}{--- 4 ---} \\
& Update epochs & \multicolumn{3}{c}{--- 8 ---} & \multicolumn{3}{c}{--- 5 ---} \\
& Discount ($\gamma$) & \multicolumn{3}{c}{--- 0.94 ---} & \multicolumn{3}{c}{--- 0.998 ---} \\
& GAE lambda ($\lambda$) & \multicolumn{6}{c}{--- 0.95 ---} \\
& Clip epsilon & \multicolumn{6}{c}{--- 0.2 ---} \\
& Entropy coef. & 0.01 & 0.05 & 0.01 & \multicolumn{3}{c}{--- 0.003 ---} \\
& Value func.\ coef. & \multicolumn{6}{c}{--- 0.5 ---} \\
& Learning rate & \multicolumn{3}{c}{--- $3 \times 10^{-4}$ ---} & \multicolumn{3}{c}{--- $5 \times 10^{-4}$ ---} \\
& Max grad norm & \multicolumn{3}{c}{--- 0.5 ---} & \multicolumn{3}{c}{--- 1.0 ---} \\
& Adam epsilon & \multicolumn{6}{c}{--- $10^{-8}$ ---} \\
\midrule
Curriculum & Episode window & 256 & 512 & 256 & \multicolumn{3}{c}{--- 128 ---} \\
\midrule
\multirow{5}{*}{State Enc.} 
& Channels & \multicolumn{3}{c}{--- [16, 32, 64] ---} & \multicolumn{3}{c}{--- [16, 32] ---} \\
& Kernel size & \multicolumn{3}{c}{--- [2, 2] ---} & \multicolumn{3}{c}{--- 5 ---} \\
& Sensor MLP out size & \multicolumn{3}{c}{--} & \multicolumn{3}{c}{--- 32 ---} \\
& Fusion MLP out size & \multicolumn{3}{c}{--} & \multicolumn{3}{c}{--- 64 ---} \\
& Activation & \multicolumn{6}{c}{--- ReLU ---} \\
\midrule
\multirow{3}{*}{Actor} 
& Hidden sizes & \multicolumn{6}{c}{--- [64, 64, 64] ---} \\
& Activation & \multicolumn{6}{c}{--- ReLU ---} \\
\midrule
\multirow{2}{*}{Critic} 
& Hidden sizes & \multicolumn{6}{c}{--- [64, 64] ---} \\
& Activation & \multicolumn{6}{c}{--- Tanh ---} \\
\midrule
\midrule
\multirow{2}{*}{Env} 
& Episode length (train) & \multicolumn{3}{c}{--- 75 ---} & \multicolumn{3}{c}{--- 1,000 ---} \\
& Episode length (test) & \multicolumn{3}{c}{--- 75 ---} & \multicolumn{3}{c}{--- 5,000 ---} \\

\bottomrule
\end{tabular}
}
\caption{Hyperparameters for \textsc{SemLTL} and the baselines across \textit{LetterWorld} and \textit{ZoneEnv} environments. Values spanning multiple columns indicate identical settings across those methods/environments.}
\label{tab:hyperparameters}
\end{table*}

\paragraph{LetterWorld.}
This environment was originally introduced by \citeauthor{DBLP:conf/icml/VaezipoorLIM21}~(\citeyear{DBLP:conf/icml/VaezipoorLIM21}) and forms a standard benchmark for LTL-conditioned RL agents. It consists of a $7\times7$ grid with 12 letters ($a$ through $l$), which are randomly placed at the beginning of each episode. Each letter appears twice in the grid, and there is always a direct path between two letters without touching any other letters. The agent can move along the cardinal directions, and the grid wraps around: if the agent goes out of bounds on one side, it is moved to the other side of the grid. The agent observes the entire grid from an egocentric perspective. Concretely, an observation is given by a tensor $o$ of shape $7\times7\times13$, where $o[x,y,z] = 1$ if the $z$th letter appears at position $(x,y)$ relative to the agent's position, and $0$ otherwise. The last channel (the $13$th position in the tensor) is zero everywhere, except for a $1$ in the centre (corresponding to the egocentric agent position). See Figure~\ref{fig:letter} for a visualisation of the environment.

\paragraph{ZoneEnv.}
The \textit{ZoneEnv} environment was also originally introduced by \citeauthor{DBLP:conf/icml/VaezipoorLIM21}~(\citeyear{DBLP:conf/icml/VaezipoorLIM21}) and has been investigated in many subsequent works \cite{DBLP:conf/nips/0002M023,DBLP:conf/iclr/JackermeierA25,DBLP:journals/corr/abs-2508-01561}.
It consists of a walled plane with 8 circular regions (i.e.\ zones) of different colours corresponding to the atomic propositions.
It contrast to the original environment, in which there are 4 colours and 2 zones per colour, we consider a more difficult version with 8 colours and only a single zone per colour.
The agent receives information about the zones via an egocentric lidar sensor.
Specifically, the 360-degree field of view around the agent is divided into 32 evenly spaced bins, and for each bin, the agent receives RGB colour and distance information for the closest zone in that direction, and a binary indicator if a zone has been detected at all.
This yields lidar observations of shape $32\times 5$.
Furthermore, the observations contain physical properties of the agent such as its acceleration, velocity, and angular velocity.
The action space is continuous and 2-dimensional, corresponding to acceleration and angular velocity.
Zone and agent positions are randomly sampled at the beginning of each episode.
See Figure~\ref{fig:zones} for a visualisation of the environment.

\subsection{Hyperparameters}
Table~\ref{tab:hyperparameters} lists hyperparameters of \textsc{SemLTL} and the baselines, including PPO parameters, curriculum window size, and details on the neural network architectures used for processing environment observations, and for the actor and critic.

\paragraph{Model architectures.}
The methods we consider all follow a similar overall model architecture:
MDP states are first encoded via a \textit{state encoder}, which is implemented as a CNN for LetterWorld and for the RGB lidar observations in ZoneEnv.
Additionally, the agent's sensor information in ZoneEnv is encoded via a single-layer MLP, and the encoded sensor information is then passed through a single-layer fusion MLP together with the encoded lidar information.
At the same time, the LTL task is encoded in a method-specific way (see below), and the resulting embedding is concatenated with the state encoding, yielding a latent representation $z$.
This is finally passed to an actor MLP that produces the parameters of the action distribution, and a critic MLP yielding a scalar value estimate.
For LetterWorld, the actor outputs the parameters of a categorical distribution, whereas it outputs the mean and standard deviation of a Gaussian distribution in the continuous case of ZoneEnv.
Additionally, \textsc{SemLTL} handles $\varepsilon$-actions as described in the paper, whereas \textsc{DeepLTL} outputs a single Bernoulli parameter $p$, which indicates whether the $\varepsilon$-action in the chosen optimal reach-avoid sequence should be taken.

\paragraph{LTL encoding networks.}
\textsc{SemLTL} encodes the current task information via the semantic embedding of the LDBA state.
The semantic embeddings are of size $762$ in LetterWorld, and $138$ in ZoneEnv.
Before passing them to the actor MLP, they are first projected to size $64$ (in LetterWorld) and $32$ (in ZoneEnv) by a simple linear layer.
For the baselines, we use the officially reported architectures and hyperparameters. 
In the case of \textsc{LTL2Action}, this is a relational graph convolutional network~\cite{DBLP:conf/esws/SchlichtkrullKB18} applied to the syntax tree of the LTL formula.
It consists of $8$ layers with shared weights and $32$-dimensional outputs.
Each layer receives the output of the previous layer concatenated with the initial one-hot node encodings.
\textsc{DeepLTL} instead encodes reach-avoid sequences using $32$-dimensional assignment embeddings, a DeepSets module~\cite{DBLP:conf/nips/ZaheerKRPSS17} with two (in LetterWorld) or one (in ZoneEnv) layers with output size $32$, and a gated recurrent unit~\cite{DBLP:conf/emnlp/ChoMGBBSB14} with hidden size $64$.

\paragraph{Other hyperparameters.}
See Table~\ref{tab:hyperparameters} for maximum episode lengths used during training and evaluation.
We use the Adam optimiser~\cite{DBLP:journals/corr/KingmaB14} throughout.

\paragraph{Training curricula.}
Our training curricula are designed to gradually expose the agent to more challenging tasks. For example, the agent is unlikely to be able to solve the task $\ltlF (\textit{red}\land\ltlF\textit{green})$ in ZoneEnv in the beginning of training, since it initially explores the environment completely randomly.
Only once the agent successfully solves single-step reachability tasks, it is feasible for it to learn to solve more complicated specifications.
We broadly follow the training curricula proposed by \citeauthor{DBLP:conf/iclr/JackermeierA25}~(\citeyear{DBLP:conf/iclr/JackermeierA25}).
However, their curricula are defined over reach-avoid sequences, whereas we directly sample LTL specifications.

The LetterWorld curriculum consists of 4 stages.
In the first stage, we sample single-step reachability tasks of the form $\ltlF a$ and reach-avoid tasks of the form $\neg a\ltlU b$.
Once the agent achieves an average success rate of 90\% across $256$ training episodes, we move to the second stage, with the same formulae substituted with disjunctions of up to length 2 (i.e.\ $\ltlF (a \lor b)$ or $\neg(a\lor b)\ltlU c$).
The third curriculum stage is achieved with a success rate of $95\%$ and consists of the same tasks as the previous stage, but always with length $2$.
Finally, the fourth curriculum stage introduces reach/avoid formulae with length and disjunctions of size up to $3$.
Further, it introduces recurrence tasks of the form $\bigwedge_{i=1}^k \ltlGF a_i \land \ltlG\left( \neg\left( \bigvee_{i=1}^l b_i \right) \right)$, where $2\leq k\leq 4$ and $0\leq l\leq 2$.

ZoneEnv follows a similar setup: the first stage consists of single-step reach tasks, which turn into two-step formulae of the form $\ltlF (a_1\land \ltlF a_2)$ in stage 2 (after achieving 90\% success rate).
Stage 3 and 4 introduce reach-avoid specifications of length $1$ and $2$, respectively.
Finally, the fifth curriculum stage consists of a mix of the above tasks with a maximum depth and disjunction size of $2$.
The final curriculum stage also introduces recurrence tasks, as in LetterWorld, and persistence tasks of the form $\ltlFG a\land \ltlG\left( \neg\left( \bigvee_{i=1}^l b_i \right) \right)$, where $0\leq l\leq 2$.

For \textsc{DeepLTL}, we use the curricula from the paper.
\textsc{LTL2Action} was originally proposed without curriculum training, but we experimented with various curriculum choices for a fair comparison.
In LetterWorld, we did not find curriculum learning to yield benefits for \textsc{LTL2Action} (in line with \cite{DBLP:conf/iclr/JackermeierA25}) and hence sample tasks from a simple distribution of the reach-avoid family with maximum depth of $3$ (i.e.\ $\neg a\ltlU (b \land (\neg c \ltlU (d \land (\neg e \ltlU f))))$.
In ZoneEnv, we follow the same curriculum as we use for $\textsc{SemLTL}$, without sampling infinite-horizon tasks (since they are not supported by \textsc{LTL2Action}).

\subsection{Literature Tasks and Results}\label{app:lit_tasks}
\begin{table*}[t]
\centering
\resizebox{0.85\textwidth}{!}{ 
\begin{tabular}{lll rr r r rr}
\toprule
\multirow{2}{*}{} & \multirow{2}{*}{} & \multirow{2}{*}{Task Specification Type} & \multirow{2}{*}{$|\mathcal Q|$} & \multirow{2}{*}{$|\delta|$} & \textsc{LTL2Action} & \textsc{DeepLTL} & \multicolumn{2}{c}{\textsc{\textbf{SemLTL}} \textbf{(Ours)}} \\
\cmidrule(lr){6-6} \cmidrule(lr){7-7} \cmidrule(lr){8-9}
& & & & & SR / $\mu_\text{acc}$ & SR / $\mu_\text{acc}$ & SR / $\mu_\text{acc}$ & $\mu_\text{states}$ \\
\midrule

\multirow{10}{*}{\rotatebox{90}{Finite-horizon}} 
& \multirow{4}{*}{\rotatebox{90}{Letter}}
& $\neg \textit{a} \ltlU (\textit{b} \land (\neg \textit{c} \ltlU (\textit{d} \land (\neg \textit{e} \ltlU \textit{f}))))$ & $6$ & $15$
& $0.73 \scriptstyle{\pm 0.20}$ 
& $\textbf{0.95} \scriptstyle{\pm 0.02}$ 
& $\textbf{0.95} \scriptstyle{\pm 0.01}$  & $4.01 \scriptstyle{\pm 0.00}$ \\

& & $(\ltlF ((\textit{a} \lor \textit{c} \lor \textit{j}) \land \ltlF \textit{b})) \land (\ltlF (\textit{c} \land \ltlF \textit{d})) \land \ltlF \textit{k}$ & $5$ & $10$
& $0.48 \scriptstyle{\pm 0.20}$ 
& $\textbf{1.00} \scriptstyle{\pm 0.00}$ 
& $\textbf{1.00} \scriptstyle{\pm 0.00}$ & $3.04 \scriptstyle{\pm 0.01}$ \\

& & $(\ltlF \textit{d}) \land (\neg\textit{f} \ltlU (\textit{d} \land \ltlF \textit{b}))$ & $14$ & $39$ 
& $0.84 \scriptstyle{\pm 0.06}$
& $0.95 \scriptstyle{\pm 0.02}$
& $\textbf{0.96} \scriptstyle{\pm 0.01}$ & $5.23 \scriptstyle{\pm 0.07}$ \\

& & $\ltlF (\textit{a} \land (\neg\textit{b} \ltlU \textit{c})) \land \ltlF \textit{d}$ & $9$ & $29$
& $0.62 \scriptstyle{\pm 0.15}$ 
& $\textbf{1.00} \scriptstyle{\pm 0.00}$ 
& $\textbf{1.00} \scriptstyle{\pm 0.00}$ & $4.20 \scriptstyle{\pm 0.03}$ \\

\cmidrule{2-9}

& \multirow{6}{*}{\rotatebox{90}{Zones}}
& $\neg(\textit{purple} \lor \textit{orange}) \ltlU (\textit{red} \land \ltlF \textit{gray})$ & $6$ & $15$
& $0.38 \scriptstyle{\pm 0.09}$ 
& $\textbf{0.85} \scriptstyle{\pm 0.04}$ 
& $\textbf{0.85} \scriptstyle{\pm 0.02}$ & $4.01 \scriptstyle{\pm 0.01}$ \\

& & $\neg \textit{gray} \ltlU ((\textit{red} \lor \textit{purple}) \land (\neg \textit{orange} \ltlU \textit{brown}))$ & $6$ & $12$
& $0.56 \scriptstyle{\pm 0.21}$ 
& $\textbf{0.91} \scriptstyle{\pm 0.02}$
& $\textbf{0.91} \scriptstyle{\pm 0.06}$ & $4.00 \scriptstyle{\pm 0.00}$ \\

& & $((\textit{green} \lor \textit{red}) \rightarrow (\neg\textit{orange} \ltlU \textit{purple})) \ltlU \textit{orange}$ & $8$ & $20$
& $0.85 \scriptstyle{\pm 0.16}$
& $\textbf{0.97} \scriptstyle{\pm 0.01}$ 
& $0.96 \scriptstyle{\pm 0.01}$ & $4.00 \scriptstyle{\pm 0.00}$ \\

& & $(\ltlF \textit{red}) \land (\neg \textit{red} \ltlU (\textit{green} \land \ltlF \textit{orange}))$ & $5$ & $10$
& $0.15 \scriptstyle{\pm 0.06}$ 
& $\textbf{0.93} \scriptstyle{\pm 0.01}$ 
& $\textbf{0.93} \scriptstyle{\pm 0.03}$  & $3.19 \scriptstyle{\pm 0.04}$ \\

& & $\ltlF (\textit{blue} \land (\neg \textit{orange} \ltlU \textit{gray})) \land \ltlF \textit{purple}$ & $5$ & $12$
& $0.53 \scriptstyle{\pm 0.20}$ 
& $0.94 \scriptstyle{\pm 0.03}$
& $\textbf{0.95} \scriptstyle{\pm 0.02}$ & $3.10 \scriptstyle{\pm 0.04}$ \\

& & $\ltlF (\textit{pink} \lor \textit{green}) \land \ltlF \textit{brown} \land \ltlF \textit{purple}$ & $4$ & $8$
& $0.61 \scriptstyle{\pm 0.23}$ 
& $\textbf{0.96} \scriptstyle{\pm 0.01}$ 
& $0.90 \scriptstyle{\pm 0.07}$  & $2.11 \scriptstyle{\pm 0.04}$ \\

\midrule
\addlinespace[-0.02ex]
\midrule

\multirow{4}{*}{\rotatebox{90}{Infinite}} 

& \multirow{2}{*}{\rotatebox{90}{L}}
& $\ltlGF (\textit{a} \land \ltlF \textit{b}) \lor \ltlGF (\textit{c} \land \ltlF \textit{d}) \land \ltlGF (\textit{e} \land \ltlF \textit{f})$ & $7$ & $15$
& n/a 
& $ 9.72 \scriptstyle{\pm 3.81}$ 
& $ \textbf{11.67} \scriptstyle{\pm 1.02}$ & $4.08 \scriptstyle{\pm 0.16}$ \\

& &  $\ltlGF \textit{a} \land \ltlGF \textit{b} \land \ltlGF \textit{c} \land \ltlGF \textit{d} \land \ltlG ( \neg \textit{e} \land \neg \textit{f}) $ & $6$ & $16$
& n/a 
& $2.57 \scriptstyle{\pm 0.31}$
& $2.58 \scriptstyle{\pm 0.18}$ & $5.44 \scriptstyle{\pm 0.05}$ \\
\cmidrule{2-9}
& \multirow{2}{*}{\rotatebox{90}{Z}}
&  $\ltlGF \textit{orange} \land \ltlGF \textit{gray}$ & $3$ & $6$
& n/a 
& $17.66\scriptstyle{\pm 8.29}$
& $\textbf{24.10}\scriptstyle{\pm11.74}$ & $3.00 \scriptstyle{\pm 0.00}$ \\

& & $\ltlGF \textit{blue} \land \ltlGF \textit{green} \land \ltlGF \textit{orange} \land \ltlG\neg\textit{pink}$ & $5$ & $13$
& n/a 
& $11.33 \scriptstyle{\pm5.26}$
& $\textbf{13.19}\scriptstyle{\pm 0.53}$  & $4.09 \scriptstyle{\pm 0.05}$  \\

\bottomrule
\end{tabular}
}
\caption{Evaluation results of \textsc{SemLTL} and baselines on tasks from the literature.
We again measure success rate (SR) for finite tasks and number of visits to accepting states ($\mu_\text{acc}$) for infinite tasks.
For task complexity, we again report $|\mathcal Q|$ (number of LDBA states), and $|\delta|$ (number of transitions) as well as the constructed LDBA states ($\mu_\text{states}$) for \textsc{SemLTL}\@.
Results are averaged over 5 seeds and 500 episodes per seed.
}
\label{tab:litresults}

\end{table*}
We consider literature tasks from \cite{DBLP:conf/iclr/JackermeierA25}.
Specifically, we use the finite tasks $\varphi_1,\ldots,\varphi_4$ for LetterWorld and $\varphi_5,\ldots\varphi_{11}$ for ZoneEnv.
Further, we use the infinite tasks $\psi_1,\ldots,\psi_2$ for LetterWorld and $\psi_3,\ldots,\psi_4$ for ZoneEnv.
The results are summarised in Table \ref{tab:litresults}, which follows the same structure as the main body table.

\begin{table*}[t]
\centering
\resizebox{0.85\textwidth}{!}{ 
\begin{tabular}{lll rr r r rr}
\toprule
\multirow{2}{*}{} & \multirow{2}{*}{} & \multirow{2}{*}{Task Specification Type} & \multirow{2}{*}{$|\mathcal Q|$} & \multirow{2}{*}{$|\delta|$} & \textsc{LTL2Action} & \textsc{DeepLTL} & \multicolumn{2}{c}{\textsc{\textbf{SemLTL}} \textbf{(Ours)}} \\
\cmidrule(lr){6-6} \cmidrule(lr){7-7} \cmidrule(lr){8-9}
& & & & & SR / $\mu_\text{acc}$ & SR / $\mu_\text{acc}$ & SR / $\mu_\text{acc}$ & $\mu_\text{states}$ \\
\midrule

\multirow{16}{*}{\rotatebox{90}{Finite-horizon}} 
& \multirow{8}{*}{\rotatebox{90}{Letter}}
& \textsc{Small} & $9$ & $23$
& $0.67 \scriptstyle{\pm 0.12}$ 
& $\mathbf{0.98} \scriptstyle{\pm 0.01}$ 
& $\mathbf{0.98} \scriptstyle{\pm 0.01}$ & $4.12 \scriptstyle{\pm 0.02}$ \\

& & \textsc{Local-Safety}[$3,2$] & $39$ & $172$
& $0.59 \scriptstyle{\pm 0.07}$ 
& $\mathbf{0.99} \scriptstyle{\pm 0.00}$ 
& $\mathbf{0.99} \scriptstyle{\pm 0.00}$ & $4.92 \scriptstyle{\pm 0.05}$ \\

& & \textsc{Local-Safety}[$3,3$] & $233$ & $1,372$
& $0.57 \scriptstyle{\pm 0.07}$ 
& \timeout
& $\mathbf{0.99} \scriptstyle{\pm 0.00}$ & $5.38 \scriptstyle{\pm 0.00}$ \\

& & \textsc{Local-Safety}[$3,4$] & $753$ & $4,944$
& $0.42 \scriptstyle{\pm 0.05}$ 
& \timeout
& $\mathbf{0.99} \scriptstyle{\pm 0.00}$ & $5.62 \scriptstyle{\pm 0.05}$ \\

& & \textsc{Global-Safety}[$3,3$] & $27$ & $116$
& $0.19 \scriptstyle{\pm 0.07}$
& $0.92 \scriptstyle{\pm 0.01}$
& $\mathbf{0.93} \scriptstyle{\pm 0.01}$ & $4.63 \scriptstyle{\pm 0.05}$ \\

& & \textsc{Global-Safety}[$4,6$] & $1,072$ & $6,414$ 
& $0.12 \scriptstyle{\pm 0.03}$ 
& \timeout
& $\mathbf{0.91} \scriptstyle{\pm 0.01}$ & $6.77 \scriptstyle{\pm 0.04}$ \\

& & \textsc{Finite-Reactive}[$5,3$] & $56$ & $338$
& $0.86 \scriptstyle{\pm 0.04}$ 
& \timeout
& $\mathbf{1.00} \scriptstyle{\pm 0.00}$ & $2.50 \scriptstyle{\pm 0.03}$ \\

& & \textsc{Finite-Reactive}[$8,2$] & $300$ & $2,115$
& $0.74 \scriptstyle{\pm 0.04}$ 
& \timeout
& $\mathbf{1.00} \scriptstyle{\pm 0.00}$ & $2.80 \scriptstyle{\pm 0.10}$ \\

\cmidrule{2-9}

& \multirow{8}{*}{\rotatebox{90}{Zones}}
& \textsc{Small} & $9$ & $23$
& $0.52 \scriptstyle{\pm 0.07}$ 
& $\mathbf{0.93} \scriptstyle{\pm 0.02}$ 
& $0.92 \scriptstyle{\pm 0.02}$ & $3.4 \scriptstyle{\pm 0.01}$ \\

& & \textsc{Local-Safety}[$3,2$] & $36$ & $155$
& $0.42 \scriptstyle{\pm 0.15}$ 
& $\mathbf{0.90}\scriptstyle{\pm 0.03}$
& $0.86 \scriptstyle{\pm 0.03}$ & $4.86 \scriptstyle{\pm 0.21}$ \\

& & \textsc{Local-Safety}[$3,3$] & $83$ & $379$
& $0.66 \scriptstyle{\pm 0.11}$ 
& $\mathbf{0.98}\scriptstyle{\pm 0.01}$
& $0.93 \scriptstyle{\pm 0.06}$ & $5.09 \scriptstyle{\pm 0.14}$ \\

& & \textsc{Local-Safety}[$3,4$] & $328$ & $1,755$
& $0.44 \scriptstyle{\pm 0.07}$ 
& \timeout
& $\mathbf{0.86} \scriptstyle{\pm 0.01}$ & $5.22 \scriptstyle{\pm 0.19}$ \\

& & \textsc{Global-Safety}[$3,3$] & $18$ & $68$ 
& $0.38 \scriptstyle{\pm 0.06}$ 
& $\mathbf{0.87} \scriptstyle{\pm 0.05}$
& $0.84 \scriptstyle{\pm 0.04}$ & $4.52 \scriptstyle{\pm 0.06}$ \\

& & \textsc{Global-Safety}[$4,8$] & $1,162$ & $6,510$ 
&  $0.39 \scriptstyle{\pm 0.05}$
& \timeout
& $\mathbf{0.79} \scriptstyle{\pm 0.06}$ & $6.73 \scriptstyle{\pm 0.18}$ \\

& & \textsc{Finite-Reactive}[$5,3$] & $28$ & $136$
& $0.87 \scriptstyle{\pm 0.06}$ 
& \timeout
& $\mathbf{0.99} \scriptstyle{\pm 0.01}$ & $2.60 \scriptstyle{\pm 0.07}$ \\

& & \textsc{Finite-Reactive}[$8,2$] & $107$ & $619$
& $0.74 \scriptstyle{\pm 0.11}$ 
& \timeout
& $\mathbf{0.98} \scriptstyle{\pm 0.00}$ & $2.96 \scriptstyle{\pm 0.16}$ \\

\midrule
\addlinespace[-0.02ex]
\midrule

\multirow{14}{*}{\rotatebox{90}{Infinite-horizon}} 
& \multirow{6}{*}{\rotatebox{90}{Letter}}
& \textsc{Small} & $7$ & $16$
& n/a 
& $6.15 \scriptstyle{\pm 1.97}$
& $\mathbf{7.13} \scriptstyle{\pm 0.51}$ & $4.76 \scriptstyle{\pm 0.09}$ \\

& & \textsc{Complex-Patrol}[$3,5$] & $39$ & $194$
& n/a 
& $4.46 \scriptstyle{\pm 0.30}$ 
& $\mathbf{4.90} \scriptstyle{\pm 0.15}$ & $4.38 \scriptstyle{\pm 0.01}$  \\

& & \textsc{Complex-Patrol}[$5,5$] & $49$ & $227$
& n/a 
& $1.78 \scriptstyle{\pm 0.26}$ 
& $\mathbf{1.83} \scriptstyle{\pm 0.10}$ & $6.44 \scriptstyle{\pm 0.03}$  \\

& & \textsc{Always-Reactive}[$3,2$] & $57$ & $236$
& n/a 
& $5.28 \scriptstyle{\pm 0.27}$ 
& $\mathbf{5.60} \scriptstyle{\pm 0.16}$ & $4.81 \scriptstyle{\pm 0.03}$ \\

& & \textsc{Always-Reactive}[$5,1$] & $38$ & $204$
& n/a 
& \timeout
& $\mathbf{3.06} \scriptstyle{\pm 0.07}$ & $7.06 \scriptstyle{\pm 0.02}$ \\

& & \textsc{Always-Reactive}[$6,1$] & $71$ & $446$
& n/a 
& \timeout
& $\mathbf{2.46} \scriptstyle{\pm 0.10}$ & $8.09 \scriptstyle{\pm 0.05}$ \\

\cmidrule{2-9}

& \multirow{8}{*}{\rotatebox{90}{Zones}}
& \textsc{Small} & $4$ & $10$
& n/a 
& $14.50 \scriptstyle{\pm 6.71}$ 
& $\mathbf{18.65} \scriptstyle{\pm 6.02}$ & $3.59 \scriptstyle{\pm 0.04}$ \\

& & \textsc{Reach-Stay}[$4$] & $18$ & $52$
& n/a 
& $703 \scriptstyle{\pm 745}$
& $\mathbf{2,763} \scriptstyle{\pm 482}$ & $6.33 \scriptstyle{\pm 0.09}$ \\

& & \textsc{Reach-Stay}[$5$] & $34$ & $116$
& n/a 
& $914 \scriptstyle{\pm 862}$
& $\mathbf{2,503} \scriptstyle{\pm 497}$ & $7.29 \scriptstyle{\pm 0.15}$ \\

& & \textsc{Complex-Patrol}[$3,5$] & $35$ & $170$
& n/a 
& $5.27 \scriptstyle{\pm 2.75}$ 
& $\mathbf{9.17} \scriptstyle{\pm 1.61}$ & $4.28 \scriptstyle{\pm 0.04}$ \\

& & \textsc{Complex-Patrol}[$5,5$] & $45$ & $201$
& n/a 
& $2.13 \scriptstyle{\pm 1.45}$ 
& $\mathbf{2.44} \scriptstyle{\pm 0.44}$ & $6.33 \scriptstyle{\pm 0.10}$ \\

& & \textsc{Always-Reactive}[$3,2$] & $57$ & $236$
& n/a 
& $8.08 \scriptstyle{\pm 3.84}$ 
& $\mathbf{10.86} \scriptstyle{\pm 1.13}$ & $4.90 \scriptstyle{\pm 0.05}$ \\

& & \textsc{Always-Reactive}[$5,1$] & $38$ & $204$
& n/a 
& \timeout
& $\mathbf{3.73} \scriptstyle{\pm 0.80}$ & $7.00 \scriptstyle{\pm 0.00}$ \\

& & \textsc{Always-Reactive}[$6,1$] & $71$ & $446$
& n/a 
& \timeout
& $\mathbf{1.43} \scriptstyle{\pm 0.39}$ & $8.00 \scriptstyle{\pm 0.00}$ \\

\bottomrule
\end{tabular}
}
\caption{Full benchmark results.
For each task family, $|\mathcal Q|$ denotes the average number of LDBA states, and $|\delta|$ the average number of transitions.
We report success rate (SR) for finite tasks and number of visits to accepting states ($\mu_\text{acc}$) for infinite tasks, as well as number of constructed LDBA states ($\mu_\text{states}$) for \textsc{SemLTL}\@.
``\timeout''~denotes failure to produce a single action within a time limit of 600s.
Results are averaged over 5 seeds and 500 episodes per seed.}
\label{tab:full_results}
\end{table*}
\subsection{Full Results}
Table~\ref{tab:full_results} lists full benchmark results across different instantiations of the task families.

\subsection{Trajectory Visualisations}
Figure \ref{fig:trajs} illustrates example trajectories produced by \textsc{SemLTL} on the ZoneEnv environment.

\begin{figure*}[b!]
    \begin{subfigure}{\textwidth}
\includegraphics[width=\textwidth]{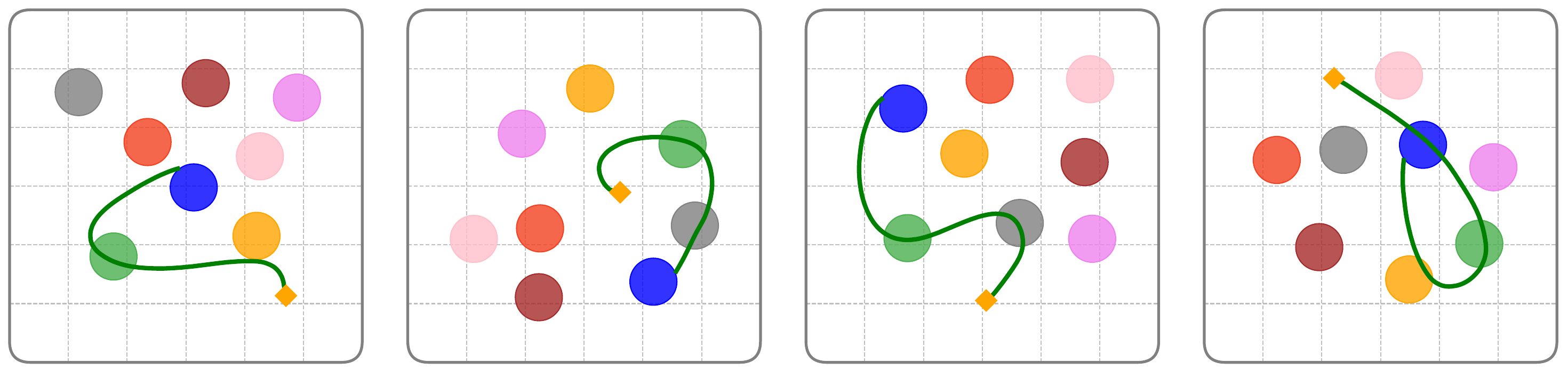}
\caption{$\ltlF (\textit{green} \land \ltlF \textit{blue})$}
\end{subfigure}

\bigskip

\begin{subfigure}{\textwidth}
\includegraphics[width=\textwidth]{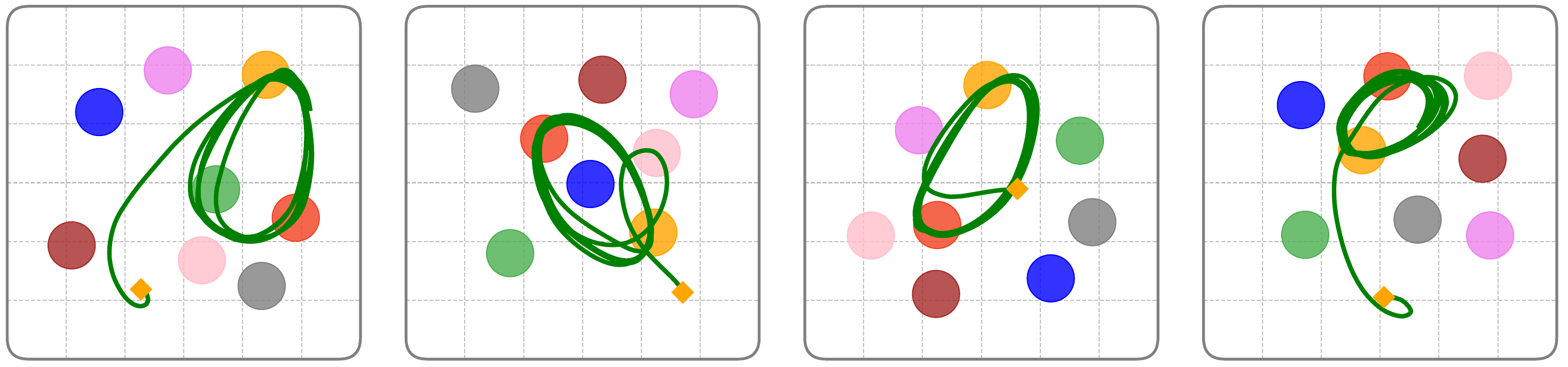}
\caption{$\ltlGF (\textit{red}) \land \ltlGF (\textit{yellow})$}
\end{subfigure}
\bigskip

\begin{subfigure}{\textwidth}
\includegraphics[width=\textwidth]{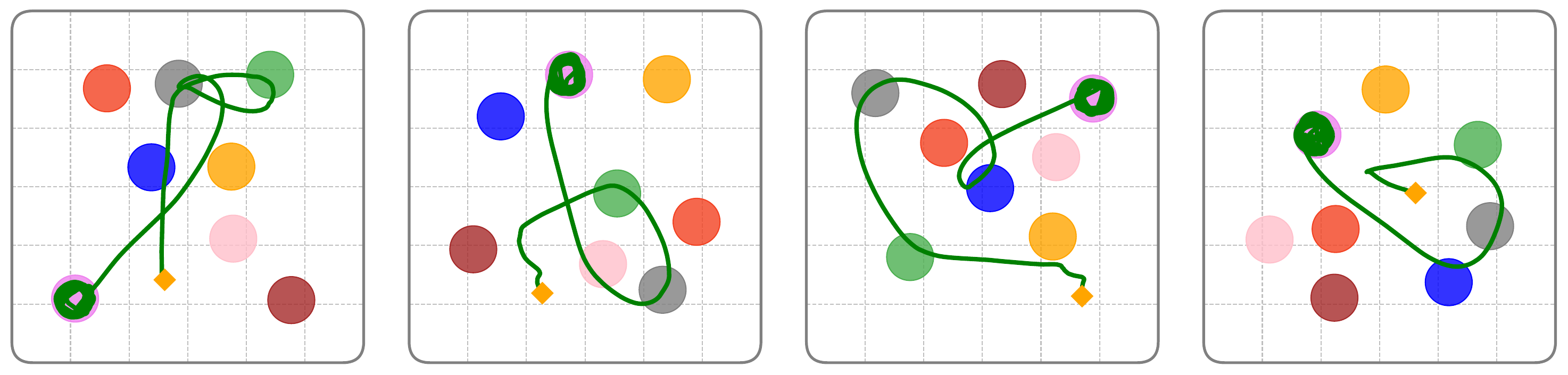}
\caption{$\textsc{Reach-Stay}[2]: \ltlF (\textit{green} \land \ltlF \textit{gray}) \land \ltlFG (\textit{purple})$}
\end{subfigure}
\bigskip

\begin{subfigure}{\textwidth}
\includegraphics[width=\textwidth]{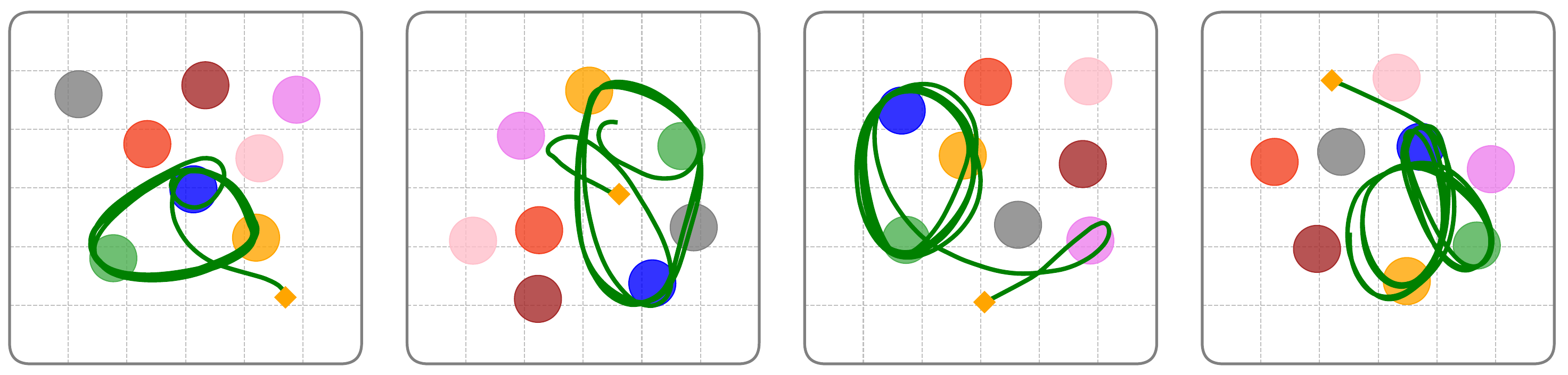}
\caption{\textsc{Always-React}[2,1]: $\ltlGF (\textit{blue}) \land \ltlG(\textit{blue} \rightarrow \ltlF \textit{yellow}) \land \ltlG(\textit{yellow} \rightarrow \ltlF \textit{green})$}
\end{subfigure}
\caption{Sample trajectories produced by \textsc{SemLTL} for the ZoneEnv environment.}
\label{fig:trajs}
\end{figure*}

\end{document}